%% file: equizero_neurips23_format.tex
\documentclass{article}

\usepackage[final]{neurips23}
\input{math_commands.tex}

\usepackage{bbm}
\usepackage{microtype}
\usepackage{graphicx}
\usepackage{tabularray}
\usepackage{hyperref}
\usepackage{url}
\usepackage{subcaption}
\usepackage{booktabs} % for professional tables

\usepackage{hyperref}

\usepackage{color}
\usepackage{multirow}
\usepackage{booktabs}
\usepackage{subcaption}

\usepackage[utf8]{inputenc} % allow utf-8 input
\usepackage[T1]{fontenc}    % use 8-bit T1 fonts
\usepackage{hyperref}       % hyperlinks
\usepackage{url}            % simple URL typesetting
\usepackage{booktabs}       % professional-quality tables
\usepackage{amsfonts}       % blackboard math symbols
\usepackage{nicefrac}       % compact symbols for 1/2, etc.
\usepackage{microtype}      % microtypography
\usepackage{xcolor}         % colors
\theoremstyle{plain}

\theoremstyle{definition}
\theoremstyle{remark}

\title{Efficient Equivariant Transfer Learning from Pretrained Models}

\author{%
  Sourya Basu~\thanks{Equal contribution.} \\
  University of Illinois at Urbana-Champaign \\
  \And
  Pulkit Katdare$~^{*}$\\
  University of Illinois at Urbana-Champaign\\
  \AND
  Prasanna Sattigeri \\
  IBM Research \\
  \And
  Vijil Chenthamarakshan \\
  IBM Research \\
  \And
  Katherine Driggs-Campbell \\
  University of Illinois at Urbana-Champaign \\
  \And
  Payel Das \\
  IBM Research\\
  \And
  Lav R. Varshney \\
  University of Illinois at Urbana-Champaign \\
}

\begin{document}

\maketitle

\begin{abstract}
Efficient transfer learning algorithms are key to the success of foundation models on diverse downstream tasks even with limited data. Recent works of \cite{basu2022equi} and \cite{kaba2022equivariance} propose group averaging (\textit{equitune}) and optimization-based methods, respectively, over features from group-transformed inputs to obtain equivariant outputs from non-equivariant neural networks. 
While \cite{kaba2022equivariance} are only concerned with training from scratch, we find that equitune performs poorly on equivariant zero-shot tasks despite good finetuning results.
We hypothesize that this is because pretrained models provide better quality features for certain transformations than others and simply averaging them is deleterious. Hence, we propose $\lambda$-\textit{equitune} that averages the features using \textit{importance weights}, $\lambda$s. These weights are learned directly from the data using a small neural network, leading to excellent zero-shot and finetuned results that outperform equitune. Further, we prove that $\lambda$-equitune is equivariant and a universal approximator of equivariant functions. Additionally, we show that the method of \cite{kaba2022equivariance} used with appropriate loss functions, which we call \textit{equizero}, also gives excellent zero-shot and finetuned performance. Both equitune and equizero are special cases of $\lambda$-equitune. To show the simplicity and generality of our method, we validate on a wide range of diverse applications and models such as 1) image classification using CLIP, 2) deep Q-learning, 3) fairness in natural language generation (NLG), 4) compositional generalization in languages, and 5) image classification using pretrained CNNs such as Resnet and Alexnet.

\end{abstract}

\section{Introduction}\label{sec:intro}

Group-equivariant deep learning leverages group equivariance as an inductive bias to design efficient and reliable neural networks. Popular examples include convolutional neural networks (CNNs) equivariant to translations \citep{lecun1989backpropagation}, group convolutional neural networks (GCNNs) equivariant to general discrete groups~\citep{cohen2016group}, and recently Alphafold2 equivariant to 3D rotations \citep{jumper2021highly}. But these methods cannot leverage pretrained models.

With the increase in open sourced large pretrained models, it is now crucial to develop efficient transfer learning algorithms that can leverage group equivariance. \citet{basu2022equi} proposed equitune, an equivariant finetuning method that uses group averaging over features extracted from pretrained models. Several other methods were proposed to get equivariant output from non-equivariant backbone architectures, e.g.\ some form of averaging~\citep{puny2021frame, atzmon2022frame} or optimization over certain proxy loss functions~\citep{kaba2022equivariance}. But these latter methods were originally designed for training from scratch and not much is known about their finetuning abilities.

Here, we find that equitune performs poorly on zero-shot tasks. Our first main contribution is to show that the optimization method of \citet{kaba2022equivariance} when used with appropriate proxy loss functions provides excellent zero-shot and finetuning performance. We call this method \textit{equizero}. 

The results from equizero suggest that pretrained models provide better quality features for
some group transformations than others. Thus, we propose $\lambda$-equitune, which learns \textit{importance weights} directly from the data and uses them to perform a weighted group averaging. We show $\lambda$-equitune outperforms equitune and is competitive with equizero for zero-shot tasks. Moreover, for finetuning, $\lambda$-equitune often outperforms both equitune and equizero. This constitutes our second main contribution.

\begin{figure}
    \centering
    \includegraphics[width=0.99\textwidth]{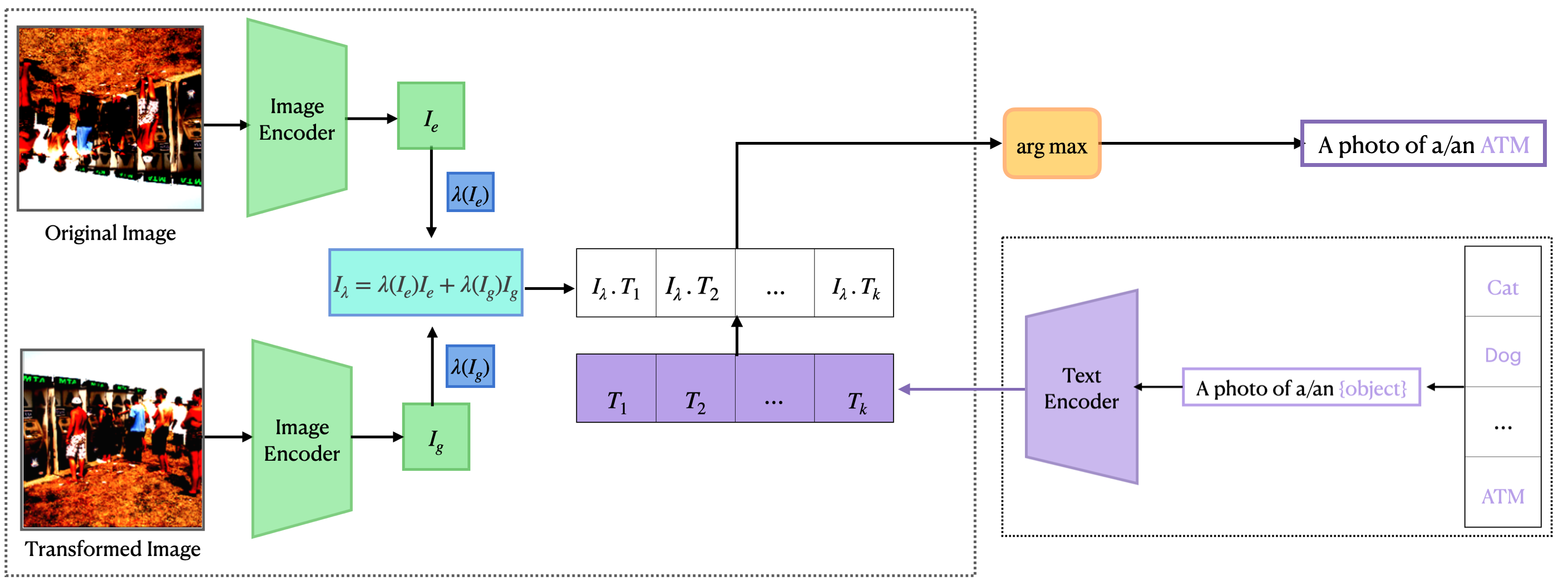}
    \caption{
    Implementation of $\lambda$-equitune on CLIP. Weighted average of image features corresponding to transformed inputs are computed, which is further used for computing text-image similarity scores.
    % Equizero is an equivariant zero-shot algorithm based on a heuristic function $l(\cdot)$ that can be used with non-equivariant pretrained models. Here, equizero used on a CLIP model is shown, where $l(\cdot)$ is the text-image similarity score. Equizero guarantees equivariance and universality, and provides better zero-shot performance than previous equivariant methods used on pretrained models.
    }
    \label{fig:equizero_clip}
\end{figure}

To validate our methods, we provide experiments on a diverse set of pretrained models and datasets. We show zero-shot performance of equizero on: 1) image classification using CLIP, 2) deep Q-learning, 3) fairness in natural language generation (NLG), and 4) compositional generalization in languages. For CLIP and deep Q-learning, we used the naturally available loss functions, namely, similarity scores and $Q$-values, respectively. For fairness in NLG and compositional generalization, we closely follow the setup of \cite{basu2022equi}, and we use \emph{regard scores}~\cite{sheng2019woman} and the negative of the maximum of the output probability distribution, respectively, as the loss function.

We first show results of $\lambda$-equitune for image classification using CLIP, finding that $\lambda$-equitune performs competitively with equizero and outperforms equitune for both zero-shot and finetuning tasks. Then, we show a simple case where finding a good proxy loss function for equizero is non-trivial: image classification using pretrained CNNs. Here, we find that equizero performs even worse than equitune but $\lambda$-equitune easily outperforms equitune and equizero. 
The organization of the paper is summarized below:
\begin{itemize}
    \item \S\ref{sec:method} provides details of $\lambda$-equitune and equizero and proves a number of their properties.
    \item \S\ref{sec:applications} provides overview of the applications considered in the experiments and how equivariance methods are used there.
    \item \S\ref{sec:experiments} provides experimental details and results for all the applications.
\end{itemize}

\section{Background}\label{sec: background}
Here we discuss relevant concepts in group equivariance and group equivariant transfer learning.
\paragraph{Group Equivariance}%\label{subsec:group_equivariance_bg}
A \textbf{group} $(G, \cdot)$ is a set $G$ accompanied by a binary operator `$\cdot$' that satisfy the axioms of a group, namely i) closure: $g\cdot h \in G$ for all $g, h \in G$; ii) associativity: $(g \cdot h) \cdot k = g \cdot (h \cdot k)$; iii) identity: there exists $e \in G$, such that $g \cdot e = e \cdot g$ for all $g\in G$; iv) inverse: for every element $g \in G$, there exists $g^{-1}$, such that $g\cdot g^{-1} = g^{-1} \cdot g = e$. We write $g\cdot h$ as $gh$ for brevity.

Given a set $\X$, we define a \textbf{group action} of $G$ on $\X$ as $\Gamma_{\X}: G \times \X \mapsto \X$ such that it satisfies two axioms, namely i) identity: $\Gamma_{\X}(e, x) = x$ for all $x \in \X$, where $e \in G$ is the identity; ii) compatibility: $\Gamma_{\X}(g, \Gamma_{\X}(h, x)) = \Gamma_{\X}(gh, x)$, for all $g, h \in G, x \in \X$. We write $\Gamma_{\X}(g, x)$ simply as $gx$ for brevity.

% group equivariance
A model $\M: \X \mapsto \Y$ is \textbf{equivariant} to $G$ under the group action of $G$ on $\X$ and $\Y$ if $\M(gx) = g\M(x))$ for all $g \in G, x \in \X$. This essentially means that any group transformation $g$ to the input $\Gamma_{\X}(g, x)$ should reflect with an equivalent group transformation of the output $\Gamma_{\Y}(g, \M(x))$. 

\paragraph{Equivariant Finetuning}%\label{subsec:equivariant_finetuning}
Recently, \citet{basu2022equi} proposed a finetuning method called equituning that starts with potentially non-equivariant model $\M$ and produces a model $\MG$ that is equivariant to $G$. Equituning converts a pretrained model into an equivariant version by minimizing the distance of features obtained from pretrained and equivariant models. The output of an equituned model is given by 
\begin{align}
    \MG (x) = \frac{1}{|G|} \sum_{g \in G}g^{-1} \M (g x). \label{eqn:equitune}
\end{align}
While the averaging in \eqref{eqn:equitune} is shown to be useful for finetuning, we find it leads to poor equivariant zero-shot learning. This could be because the pretrained model outputs high quality features only for some of the transformed inputs. Hence, averaging them directly leads to low quality zero-shot performance. This can be avoided by using weighted averaging as discussed in Sec.~\ref{sec:method}.

\paragraph{Optimization-Based Canonical Representation}%\label{subsec:equivariant_opt}
On the other hand, \cite{kaba2022equivariance} show that group equivariant output, $\MO(x)$, can be obtained by optimizing a (non-equivariant) loss function $l(\M(gx))$ with respect to group elements $g \in G$ for any $x\in \X$ as shown below.
\begin{align}
    \MO(x) = g^{-1}_{*}\M (g_{*}x),   \label{eqn:kaba_canonicalization}
\end{align}
where $g_* = \argmin_{g \in G} l(\M (gx))$ with $l:\Y \mapsto \R$ being an injective proxy loss function and assuming the minima is unique. However, the purpose of this formulation in \cite{kaba2022equivariance} is only to obtain an equivariant representation for training from scratch. Moreover, no competitive zero-shot or finetuning performance is obtained in \cite{kaba2022equivariance} using this method. We show that the choice of $l$ plays an important role in its zero-shot and finetuning performance, even outperforming equituning. This method is obtained as a special case of $\lambda$-equituning introduced next.\\
\paragraph{Additional Related Works} Group equivariance plays a key role in geometric deep learning~\citep{bronstein2021geometric} for designing efficient domain specific neural networks. Several elegant architectures have been proposed for equivariant image classification~\citep{cohen2016group, cohen2016steerable, romero2020group}, reinforcement learning~\citep{mondal2020group, van2020mdp, mondal2022eqr, wang2022equivariant}, graph~\citep{satorras2021en, keriven2019universal, gasteiger2021gemnet} and mesh~\citep{de2020gauge, he2021gauge, basu2022equivariant} processing, natural language processing~\citep{gordon2019permutation, li2022equivariant}, and data generation~\citep{dey2020group}. These architectures need to be trained from scratch, which is not always desirable.

% variants of group averaging
\textit{Frame} averaging produces equivariant output from non-equivariant architecture backbones~\citep{puny2021frame, atzmon2022frame, duval2023faenet}. Most work here focuses on finding good frames, which are equivariant subsets of groups, for specific groups, and not for general groups. And not much is known about their performance with pretrained models. \citet{kaba2022equivariance} also give a canonicalization-based method that uses an equivariant auxiliary network for constructing equivariant networks out of non-equivariant backbones and is used for training from scratch. But this work requires an additional equivariant network and appropriate trainable parameterization of the group actions, which is presented only for certain groups of interest. Further, this work is not concerned with equivariant performance of pretrained models. Zero-shot group equivariance was also recently used by \cite{muglichequivariant} for zero-shot coordination in partially observable Markov decision processes (POMDPs). In contrast, our work aims to be provide efficient equivariant transfer learning algorithms that are general in terms of considered tasks and groups, and does not require additional equivariant networks.

\section{$\lambda$-Equitune}\label{sec:method}

We propose $\lambda$-equitune, where unequal weights are assigned to features obtained from transformed inputs. This is a simple generalization of equitune in \eqref{eqn:equitune} and the optimization-based approach in \eqref{eqn:kaba_canonicalization}, where the goal is to assign higher values to better features. Like these previous methods, $\lambda$-equitune is equivariant and a universal approximator of equivariant functions.

The main idea of $\lambda$-equitune is that given a pretrained model $\M$, the features $\M(g x)$ for any fixed $x$ are not all equally important for all $g \in G$. We denote by $\lambda (gx)$ the \textit{importance weight} of feature $\M(gx)$ for $g \in G, x \in \X$. We assume $G$ is finite, just as in \cite{basu2022equi}. Suppose $\lambda: \X \mapsto \R^+$ is known a priori, and denote the $\lambda$-equituned model as $\MU$. Then we want to minimize 
\begin{equation}
\begin{aligned}
\min_{\MU (x)} \quad & \sum_{g\in G}\norm{\lambda(gx)\M(g x) - \MU(g,x)}_2^2\\
\textrm{s.t.} \quad & \MU(g x) = g \MU(x) \textrm{ for all } g \in G.\\
\end{aligned}\label{eqn:un-equi-tuning-loss-func}
\end{equation}
We call the solution to \eqref{eqn:un-equi-tuning-loss-func} as $\lambda$-equitune, given by
\begin{align}
    \MU (x) = \frac{1}{\sum_{g \in G}\lambda(gx)} \sum_{g \in G}^{|G|}g^{-1} \lambda(gx) \M (g x). \label{eqn:un-equituning-solution}
\end{align}
\subsection{Special Cases}\label{subsec:equizero}
When $\lambda (gx) = 1$ for $g \in G$, then \eqref{eqn:un-equituning-solution} is the same as \eqref{eqn:equitune}. Further, we get \eqref{eqn:kaba_canonicalization} when $\lambda$ is an indicator function $\lambda(gx) = \mathbbm{1}_{\{g = g_{*}\}}$, where $g_* = \argmin_{g \in G} l(\M (g x))$ with $l:\Y \mapsto \R$ such that the minimization is well defined. We use $\MGZ$ to denote the equizero model.

The first main contribution of our work is experimental. We show there are good loss functions for equizero for several diverse tasks that easily outperform equitune by choosing the best features. 

The second main contribution is to show that $\lambda$-equitune outperforms equitune and is competitive with equizero even when good loss functions are known. Moreover, when the loss functions are not trivial, equizero performs even worse than equitune, but $\lambda$-equitune easily outperforms both.

\subsection{Canonical-$\lambda$-Equitune}\label{subsec:canon_lambda_equitune}
Here, we provide an extension of the $\lambda$-equitune algorithm of \eqref{eqn:un-equituning-solution} to continuous groups.
The idea is to combine the canonicalization method of \cite{kaba2022equivariance} with $\lambda$-equitune leading to an expressive equivariant network with weighted averaging over features with different group actions applied to them.

\begin{definition}[Canonical-$\lambda$-equitune]
    Given a (continuous) group $G$, a non-equivariant function $M:X\mapsto Y$, and an equivariant auxiliary function (from the setting of \cite{kaba2022equivariance}) $h: X \mapsto G$, lambda functions $\lambda: X \mapsto R^+$, and a set of group elements $\Theta = \{\theta_1, \ldots, \theta_k\}$, i.e. $\theta_i \in G$, we define the canonical-$\lambda$-equitune operators as
    \begin{align}
        M^{\lambda}_{G, equi}(x) &= \frac{1}{\sum_{\theta \in \Theta} \lambda(\theta h(x)^{-1}g^{-1}gx)} \sum_{\theta \in \Theta} \lambda(\theta h(x)^{-1}x) h(x) M(\theta h(x)^{-1}x \label{eqn:lambda_canon_equitune}\\
        M^{\lambda}_{G, inv}(x) &= \frac{1}{\sum_{\theta \in \Theta} \lambda(\theta h(x)^{-1}g^{-1}gx) } \sum_{\theta \in \Theta} \lambda(\theta h(x)^{-1}x) M(\theta h(x)^{-1}x). \label{eqn:lambda_canon_invtune}
    \end{align}

\end{definition}
In Thm.~\ref{thm:lambda_canon_equivariance}, we show that the canonical-$\lambda$-equivariant network is equivariant to the group $G$
\begin{theorem}\label{thm:lambda_canon_equivariance}
    $M^{\lambda}_{G, equi}(x)$ and $M^{\lambda}_{G, inv}(x)$ are, respectively, equivariant and invariant to $G$.
\end{theorem}

\subsection{Properties}\label{subsec:properties}

Now we show in Thm.~\ref{thm:equizero_equivariance} that \eqref{eqn:un-equituning-solution} is equivariant with respect to $G$. 
% equivariance
\begin{theorem}[Equivariance]\label{thm:equizero_equivariance}
    $\MU$ defined in \eqref{eqn:un-equituning-solution} is equivariant to $G$.
\end{theorem}
We define a universal approximator in Def.~\ref{def:universality}. Then, Thm.~\ref{thm:equizero_universality} proves that $\lambda$-equitune is a universal approximator of equivariant functions for groups where $\norm{g}=1$ for $g\in G$. This includes a wide range of groups including the permutation group, the $SO(n)$ groups of special orthogonal groups, etc. This condition is the same as the class of groups considered in \cite{basu2022equi}.

\begin{definition}[Universal approximator]\label{def:universality}
A model $\M: \X \mapsto \Y$ is a universal approximator of a continuous function $f: \X \mapsto \Y$ if for any compact set $\K \subset \X$ and $\epsilon > 0$, there exists a choice of parameters for $\M$ such that $\norm{f(x) - \M(x)} \leq \epsilon$ for all $x \in \K$.
\end{definition}

% universality
\begin{theorem}[Universality]\label{thm:equizero_universality}
    Let $f_G:\X \mapsto \Y$ be any continuous function equivariant to group $G$ and let $\lambda: \X \mapsto \R^+$ be any positive scalar function. And let $\M: \X \mapsto \Y$ be a universal approximator of $\frac{f_G}{\lambda}$. Here $\X$ is such that if $x \in \X$, then $gx \in \X$ to ensure the equivariance of $f_G$ is well-defined. Then, $\MU$ is a universal approximator of $f_G$.
\end{theorem}

% add properties of equizero
% computational-complexity
\paragraph{Computational Complexity}
Note that equitune, equizero, and $\lambda$-equitune have the same compute complexity. In practice, in Tab.~\ref{tab:compute_equitune_vs_equizero} we find that equitune and equizero have exactly the same time and memory consumption, whereas $\lambda$-equitune takes a little more time and memory because of the additional $\lambda$ network. We illustrate on RN50 and ViT-B/32 models of CLIP using the same text encoder, but different image encoders. RN50 uses a Resnet50 based image encoder, whereas ViT-B/32 uses a vision transformer based image encoder. 
\input{compute_complexity_comparison.tex}

% beyond zero-shot generalization
\paragraph{Beyond Zero-Shot Learning} Let us emphasize that even though the equizero model in \eqref{eqn:kaba_canonicalization} is not differentiable due to the $\argmax$, we can still use simple gradient estimators known in the literature. One popular estimator is the straight-through estimator~\citep{bengio2013estimating}, where the equizero output in \eqref{eqn:kaba_canonicalization} would be written as $\MGZ(x) = \M(x) + (\MGZ(x) - \M(x))$.\texttt{detach()}, where \texttt{detach()} indicates that no gradient flows through the term $(\MGZ(x) - \M(x))$. In practice, we found it to be slightly better to use $\MG(x)$ instead of $\M(x)$ and write $\MGZ(x) = \MG(x) + (\MGZ(x) - \MG(x))$.\texttt{detach()}. \S\ref{subsec:compositional_generalization_exp} illustrates few-shot learning and finetuning using equizero and compares with equituning.

%%%%%%%%%%%%%%%%%%%%%%%%%%%%%%%%%%%%%%%%%%%%%%%%%%%%%%%%%%%%%%%%%%%%%%%%%%%%%%%%%%%%%%%
%%%%%%%%%%%%%%%%%%%%%%%%%%%%%%  Applications  %%%%%%%%%%%%%%%%%%%%%%%%%%%%%%%%%%%%%%%%%
%%%%%%%%%%%%%%%%%%%%%%%%%%%%%%%%%%%%%%%%%%%%%%%%%%%%%%%%%%%%%%%%%%%%%%%%%%%%%%%%%%%%%%%
\section{Applications}\label{sec:applications}
First we present several applications in \S\ref{subsec:equizero_applications} where finding a loss function for equizero is easy. This naturally leads to excellent equivariant zero-shot results outperforming equitune. Then, in \S\ref{subsec:lambda_applications} we provide two applications in image classification to show 1) the benefits and drawbacks of equizero compared to equitune and $\lambda$-equitune, and 2) that $\lambda$-equitune consistently performs well avoiding the drawbacks of equizero. 

\subsection{Equizero Applications}\label{subsec:equizero_applications}
Here we provide applications where equizero achieves excellent zero-shot performance, namely: 1) deep Q-learning, 2) fairness in NLG, and 3) compositional generalization in languages.

\paragraph{Equizero Reinforcement Learning}%\label{subsec:equi0_rl_app}
% recently rl has started using symmetries, none use pretrained models, we do using equizero
Recent works, such as \cite{van2020mdp, mondal2020group}, have developed RL algorithms that leverage symmetries in the environments that helps improve robustness and sample efficiency. But no existing work efficiently uses group equivariance on top of pretrained RL models. Here we apply equizero and equitune on top of pretrained models inspired from the group symmetries found in \citet{van2020mdp}. We find that equitune outperforms non-equivariant pretrained models but equizero outperforms both. We simply use the $Q$-values as the proxy loss function as described in \S\ref{subsec:app_eq_DQL} with more details.

\paragraph{Group-Theoretic Fairness in NLG}%\label{sec:equinlg}
% general background
We seek to reduce the social biases inherent in language models (LMs), focusing on GPT2~\citep{radford2019language}. We consider the group-theoretic fairness setup of \citet{sheng2019woman} and \citet{basu2022equi}. We take the sets of demographics, namely [`man', `woman'], [`straight', `gay'], and [`black', `white']. For each demographic group, \citet{sheng2019woman} proposed two tasks, called \textit{respect task} and \textit{occupation task}, where each task consists of five context phrases. A language model (LM) is given these contexts to generate sentences. These generated sentences are then classified as `positive', `negative', `neutral', or `other' by a \textit{regard classifier} also proposed by \citet{sheng2019woman}. These outputs are called regard scores. A regard classifier is a BERT-based model similar to a sentiment classifier but more specific for fairness tasks. We use equizero using the regard scores as the proxy loss function, so as to maximize positivity in the generated texts, while guaranteeing group-theoretic fairness. To that end, we propose two algorithms that help in reducing existing biases across demographic groups. 

\textit{EquizeroLM and R-EquizeroLM:} \citet{basu2022equi} define EquiLM and R-EquiLM, that use a sequential version of equitune to perform group transformed averaging to achieve fairness across demographic groups. While EquiLM and R-EquiLM generate debiased outputs, they do not produce positive regard scores, which is desirable to reduce toxicity in generated text. We propose EquizeroLM and R-EquizeroLM which use equizero to maximize regard scores, ensuring both fair and less toxic. Further details provided in \S\ref{subsec:app_fairness_in_NLG}.

\paragraph{Zero-Shot Compositional Generalization}% \label{subsec:equi0_scan_app}
We show compositional generalization capabilities of equizero on the SCAN dataset~\citep{lake2018generalization}. SCAN measures compositionality using a language-to-action translation task. E.g., if the model learns that the phrase ``Jump", ``Run", ``Run Twice" translate to the actions ``JUMP", ``RUN", ``RUN RUN" from the train set, then, SCAN tests whether the model also learns that ``Jump Twice" translates to ``JUMP JUMP". Such a reasoning however common in human beings is hard to find in language models.

We apply equizero on two tasks in SCAN, \textit{Add Jump} and \textit{Around Right}. \citet{gordon2019permutation} solved these tasks by constructing sophisticated group equivariant networks from scratch and training them. \citet{basu2022equi} used the same group actions as \citet{gordon2019permutation}, but used equituning on pretrained non-equivariant models for a few iterations and obtained comparable results. But, as we note, equitune has poor zero-shot performance. We show that equizero using negative of maximum probability from the output as the loss function gives much better zero-shot performance. Using gradient estimators described in \S\ref{subsec:properties}, we also compare the finetuning performance of equizero against equitune. Details of group actions used are given in \S\ref{subsec:app_application_comp_generalizatiion}

\subsection{$\lambda$-Equitune Applications}\label{subsec:lambda_applications}
Here we consider two important applications: CLIP-based and CNN-based image classification. For CLIP, it is easy to find a loss function for equizero that provides better results than equitune. But for the simple case of CNN-based classification it is non-trivial to find such a loss function. Since $\lambda$-equitune does not require a loss function, it performs better than equitune in both cases. Equizero only performs well for CLIP, but fails miserably for CNN-based classification.

\paragraph{CLIP-Based Image Classification}
CLIP is a pretrained model consisting of image and text encoders that give impressive zero-shot classification performance across variety of unseen datasets. But, in Fig.~\ref{fig:test_imagenet} and \ref{fig:test_cifar100}, we find that CLIP is not robust to simple transformations such as rotation by multiples of $90^{\circ}$ or random flips. This trend is seen across different image encoders like RN50, RN101 \citep{radford2021learning}, ViT-B/32 and ViT-B/16 \citep{DBLP:conf/iclr/DosovitskiyB0WZ21}. This can be avoided by making the model in/equi-variant to such transformations, e.g., by using equitune. But we show in \S\ref{subsec:equiCLIP_exp} that equitune does not produce good zero-shot performance. We show in \S\ref{subsec:equiCLIP_exp} that using equizero with image-text similarity score  as loss function provides much better zero-shot results than equitune. Later, in \S\ref{subsec:equiCLIP_exp}, we show that $\lambda$-equitune achieves better zero-shot performance than equitune without any loss function. Finally, when finetuned, we find that $\lambda$-equitune tends to perform the best, possibly because it does not need gradient estimators like equizero, and because it uses weighted averaging to obtain better features than equitune.

\paragraph{CNN-based image classification}
For image classification using pretrained CNNs such as Resnet and Alexnet, we note that finding a good loss function for equizero is non-trivial. As such, we consider two loss functions 1) negative of the maximum probability as it worked well with the SCAN task in \S\ref{subsec:equizero_applications} and 2) the entropy of the output distribution since it correlates with the confidence of the model~\citep{wangidk}. But, equizero with these loss functions performs even worse than equitune. Further, we find that $\lambda$-equitune easily outperforms both equitune and equizero.

\section{Experiments}\label{sec:experiments}

Here, we provide experimental results for equizero and $\lambda$-equitune in \S\ref{subsec:equivariant_zero_shot} and \S\ref{subsec:exp_equivariant_zero_shot}, respectively, for all the applications described in \S\ref{sec:applications}. Additional experiments for canonical-$\lambda$-equitune are provided in \S.~\ref{subsec:app_additional_results_canon_lambda_equitune}. The code for this paper is available at \url{https://github.com/basusourya/lambda_equitune}.

\input{rl_equizero.tex}

\subsection{Zero-Shot Performance using Equizero}\label{subsec:equivariant_zero_shot}

\subsubsection{Equizero Reinforcement Learning}\label{subsec:equiRL_exp}
\textbf{Experimental Setting:} We first pretrain Deep Q-learning nets (DQNs) for each of the Gridworld, Cartpole, and Acrobot environments using the default architecture from \citet{raffin2021stable} with 103k parameters. We pretrained all the models using a learning rate $10^{-4}$. We used training time steps as 100k, 100k, and 70k for Gridworld, Cartpole, and Acrobot, respectively. These number of steps were chosen to obtain the best models by varying the time steps from 50k to 100k in multiple of 10k for a fixed seed.\\
\textbf{Results and Observations:} 
Fig.~\ref{fig:rl} show the evaluation performance of equizero and compare it with equituning and non-equivariant pretrained models. We find that equituning performs better than non-equivariant models and equizero outperform both of them. The results are over five seeds.

\subsubsection{Fairness in Natural Language Generation}\label{subsec:fairness_exp}
\input{equinlg_respect.tex}
\textbf{Experimental Setting}: We use GPT2 \citep{radford2019language}, with 117M parameters as our pretrained model. We consider the lists of demographic groups [`man', `woman'], [`white', `black'], and [`straight', `gay' ]. We compare our method against EquiGPT2 and R-EquiGPT2 \citep{basu2022equi}. For the equizero models, we use beam length $m$ as 5, as described in \S\ref{subsec:app_fairness_in_NLG}. We limit the sentence lengths to be 15 for all models. We generated 500 sentences for each of the models and demographic groups by varying the seeds for both respect and occupation context. \\
\textbf{Results and Observations}: Fig.~\ref{fig:equinlg_respect} and \ref{fig:equinlg_occupation} compare the regard scores for all the considered models. We observe that equizero models are not only able to debias among various demographics like `man' and `woman', but it also reduces the toxicity/ negativity of the scores. Debiasing is seen from the equality of the scores amongst the demographics considered for equivariance. And reduction in toxicity is observed by noticing that the regard scores are more positive. Like equituning~\citep{basu2022equi}, equizero models show high quality of generated texts. Sample generations from the demographic groups [`straight', `gay'] are shown in Tab.~\ref{tab:EquizeroGPT2_gender_respect}, \ref{tab:EquiGPT2_gender_respect}, \ref{tab:REquiGPT2_gender_respect}, and \ref{tab:GPT2_gender_respect} for all the models. Note in Tab.~\ref{tab:EquizeroGPT2_gender_respect} and \ref{tab:EquiGPT2_gender_respect}, that even with perfect equivariance, where the word `straight' simply gets replaced by `gay', the regard scores are very different. This shows the presence of bias in the regard classifier itself as was also observed by \citet{basu2022equi}.

\subsubsection{Compositional Generalization using Equizero}\label{subsec:compositional_generalization_exp}
\input{SCAN_lstm_equizero_zeroshot.tex}
\textbf{Experimental Setting:} 
% datasets and models % pretraining procedure % finetuning procedure
We evaluate the performance of our algorithm on the SCAN Dataset \citep{gordon2019permutation} on the \textit{Add Jump} and \textit{Around Right} tasks. 
All the recurrent models (RNN, LSTM, GRU) and their equivariant counterparts contain a single hidden layer of 64 units. For all training processes, we use the Adam optimizer~\citep{kingma2015adam} and teacher-forcing ratio 0.5~\citep{williams1989learning}. For pretraining and finetuning, we use 200k and 10k iterations respectively. For pretraining, we use a learning rate of $10^{-4}$, whereas for finetuning, we used learning rates $10^{-5}$ and $2\times 10^{-5}$ for \textit{Add Jump} and \textit{Around Right}, respectively. Results are over three seeds.\\
\textbf{Observation and Results:} 
%zeroshot performance
Tab.~\ref{tab:equizero_zeroshot_LSTM_scan} shows that equizero outperforms equituned and non-equivariant models on zero-shot performance with LSTMs. In \S\ref{subsec:additional_results}, we observe similar results for RNNs and GRUs in Tab.~\ref{tab:equizero_zeroshot_RNN_scan} and \ref{tab:equizero_zeroshot_GRU_scan}, respectively. 
We use gradient estimators discussed in \S\ref{subsec:properties} for performing few-shot learning using equizero.  For few-shot learning, we find in Fig.~\ref{fig:equituning_vs_equizero} in \S\ref{subsec:additional_results} that equizero is competitive with equitune for small iterations, but as the number of steps increase, equitune is the better algorithm. This is expected since the gradients are computable for equitune, but only approximated in equizero. Tab.~\ref{tab:equizero_10k_LSTM_scan}, \ref{tab:equizero_10k_GRU_scan}, and \ref{tab:equizero_10k_RNN_scan} in \S\ref{subsec:additional_results} provide the results for finetuning using equitune and equizero for 10k iterations. We find the equitune is slightly better when finetuning for 10k iterations. This shows equizero is better for zero-shot and few-shot learning, but for large iterations, equitune is preferable.

%%%%%%%%%%%%%%%%%%%%%%%%%%%%%%%%%%%%%%%%%%%%%%%%%%%%%%%%%%%%%%%%%%%%%%%%%%%%%%%%%%%%%%%%%%%%%%%%%%%%%%%%%%%%%%%
%%%%%%%%%%%%%%%%%%%%%%%%%%%%%%%%% \lambda-equitune %%%%%%%%%%%%%%%%%%%%%%%%%%%%%%%%%%%%%%%%%%%%%%%%%%%%%%%%%%%%
%%%%%%%%%%%%%%%%%%%%%%%%%%%%%%%%%%%%%%%%%%%%%%%%%%%%%%%%%%%%%%%%%%%%%%%%%%%%%%%%%%%%%%%%%%%%%%%%%%%%%%%%%%%%%%%
\subsection{Zero-Shot and Finetuning Performance using $\lambda$-Equitune}\label{subsec:exp_equivariant_zero_shot}
% \subsubsection{Zero-Shot Equivariance on CLIP}\label{subsec:equiCLIP_exp}
\subsubsection{Equi/Invariant Image Classification using CLIP}\label{subsec:equiCLIP_exp}
\input{clip_imagenet.tex}
\textbf{Experimental Setting:} 
We first perform zero-shot image classification on Imagenet-V2 and CIFAR100 using CLIP. We use two transforms, random $90^{\circ}$ rotations and flips, for testing their robustness. We encode class labels using the 80 text templates provided in \citet{radford2021learning}~\footnote{Obtained from \url{https://github.com/openai/CLIP}}.

We then evaluate finetuning capabilities of equitune, equizero, and $\lambda$-equitune on CIFAR100 with random $90^{\circ}$ rotations. Here we choose $\lambda$ to be a two-layered feed-forward network, which is described in detail along with learning rates used, in \S\ref{subsec:app_lambda_equitune_image_classification}. The input to this $\lambda$-network are the features from a frozen CLIP encoder. First, the $\lambda$-network is trained with the model kept frozen. Then, while finetuning the model, the $\lambda$ network is frozen. We show results for both Resnet and ViT backbone trained for 1000, 2000, 3000, 4000 finetuning steps. \\ \\
\textbf{Results and Observations:} Fig.~\ref{fig:clip} shows test accuracies for Imagenet-V2 with random $90^{\circ}$ rotations and flips. We observe that the pretrained model's performance reduces drastically when transformations are applied to the dataset. Whereas both equitune and equizero are relatively robust to such transformations. Moreover, equizero outperforms both equitune and the pretrained model. Similar observations are made for CIFAR100 in Fig.~\ref{fig:clip_cifar} in \S\ref{subsec:additional_results}.

In Fig.~\ref{fig:clip_finetune_a} and ~\ref{fig:clip_finetune_additional} we plot the test accuracies of $\lambda$-equitune on CIFAR100 for both variants of Resnet and ViT backbones. We observe that $\lambda$-equitune performs better than both equitune and equizero (with finetuning) on Resnets. On ViT-B/16, we observe that $\lambda$-equitune easily outperforms both equitune and equizero (with finetuning). On ViT-B/32, we find that $\lambda$-equitune outperforms equitune but equizero outperforms both equitune and $\lambda$-equitune. Thus, $\lambda$-equitune performs competitively with equizero, even in applications where good loss functions are known.
\subsubsection{Equi/Invariant Image Classification using Pretrained CNNs}\label{subsubsec:exp_unequitune_classification}
\textbf{Experimental Setting:}
We now evaluate the performance of $\lambda$-equitune on rot90-CIFAR10 dataset. Here, each image is rotated randomly by a multiple of $90^{\circ}$. We use pretrained Alexnet and Resnet, trained extensively over CIFAR100. For the $\lambda$-equituning, we choose $\lambda$ as a two layered feed-forward network with a hidden layer of dimension 100 with input as features extracted by this pretrained Alexnet or Resnet. Along with $\lambda$, we also perform linear probing wherein last two layers for the classification problem is being learned using a learning rate of $10^{-3}$. 
\input{classification_finetuning}\\
\textbf{Observation and Results:} In Fig.~\ref{fig:finetune_classification} and \ref{fig:finetune_classification_additional} we see that $\lambda$-equitune outperforms equitune and equizero. Moreover, equizero performs even worse than equitune. This trend is consistent across both Alexnet and Resnet pretrained modules.

\section{Limitations, Societal Impact, and Conclusion}\label{sec:conclusion}
\textbf{Limitations and Societal Impact:} Our results focus on finite groups; extension to continuous groups requires further work in parameterization of continuous groups (cf.~\citet{benton2020learning}) or group decomposition (cf.~\citet{basu2021autoequivariant, maileequivariance}). Our work on fairness in NLG aims to debias foundation models and thereby lead to positive societal impact. But we use equality and neutral sets obtained from previous human-made works. We believe there is scope for optimizing the design of such sets using neural networks for more general demographic groups.\\
\textbf{Conclusion:}
We present $\lambda$-equitune and its special case equizero that outperform equitune on equivariant zero-shot and finetuning tasks. We show that both methods are equivariant, universal, are computationally efficient, and are easy to implement. Equizero performs well when good proxy loss functions are available for downstream tasks, which we show is easy to find across several diverse tasks. $\lambda$-equitune outperforms equitune and is competitive with equizero without the need for any proxy loss. We consider diverse tasks: i) deep Q-learning, ii) fairness in NLG, iii) compositional generalization in languages, iv) CLIP-based classification, and v) CNN-based classification. Experimental results validate the superiority of $\lambda$-equitune and equizero over equitune on zero-shot and finetuning tasks.

\section*{Acknowledgment}
Discussion with Moulik Choraria on the $\lambda$-equitune finetuning for CLIP is appreciated. 
A portion of the work was supported by the Department of Energy (DOE) award (DE-SC0012704).

%\bibliography{neurips2023/neurips23}
\bibliography{neurips23}
\bibliographystyle{neurips23}

\appendix
\onecolumn

\section*{Appendix}
\section{Proofs}\label{sec_app:proofs}
\begin{proof}[Proof of Thm.~\ref{thm:equizero_equivariance}]
We want to show $\MU(h x) = h \MU(x)$ for all $x \in \X$ and $h \in G$. From the definition of $\MU$ in \eqref{eqn:un-equituning-solution}, we have $\MU(h x) = \frac{1}{\sum_{g \in G}\lambda(gx)} \sum_{g \in G}g^{-1} \lambda(gx) \M (g x)$. We have

\begin{align}
    \MU(h x) &= \frac{1}{\sum_{g \in G}\lambda(ghx)} \sum_{g \in G}^{|G|}g^{-1} \lambda(g hx) \M (g hx) \nonumber \\
            &= \frac{1}{\sum_{g \in G}\lambda(gx)} \sum_{gh \in G}^{|G|}h (gh)^{-1} \lambda((g h)x) \M ((g h)x)\nonumber \\
            &= \frac{1}{\sum_{g \in G}\lambda(gx)} \sum_{g \in G}^{|G|}h (g)^{-1} \lambda(gx) \M (gx)\label{eqn:equiv_proof_2}\\
            &= h \MU(x), \label{eqn:equivariance_proof}
\end{align}
where \eqref{eqn:equiv_proof_2} follows because $g \in G$ implies $gh \in G$ for any $h \in G$.
\end{proof}

\begin{proof}[Proof of Thm.~\ref{thm:equizero_universality}]
Since $\M$ is an universal approximator of $\frac{f_G}{\lambda}$, we have that for any compact set $\K \in \X$ and $\epsilon >0$, there exists a choice of parameters of $\M$ such that $\norm{\frac{f_G(x)}{\lambda(x)} - \M(x)} \leq \epsilon$ for all $x \in \K$.

Similar to \citet{yarotsky2022universal}, we first define $\K_{sym} = \bigcup_{g \in G} g \K$. Note that $\K_{sym}$ is also a compact set and $\K_{sym} \subset \X$. Thus, we also have a choice of parameters of $\M$ such that $\norm{\frac{f_G(x)}{\lambda(x)} - \M(x)} \leq \epsilon$ for all $x \in \K_{sym}$ for the same $\K$ and $\epsilon > 0$ defined above.

Hence, from the definition of $\MU$ in \eqref{eqn:un-equituning-solution}, we have 
\begin{align}
    \norm{f_G(x) - \MU(x)} &= \norm{\frac{1}{\sum_{g \in G}\lambda(ghx)} \sum_{g \in G} (g^{-1} f_G(g x) - g^{-1} \lambda(gx) \M (g x))} \label{eqn:univ_proof_1}\\
                    &\leq \frac{1}{\sum_{g \in G}\lambda(ghx)} \sum_{g \in G} \lambda(gx) \norm{\frac{f_G(gx)}{\lambda(gx)} -  \M (g x)} \label{eqn:univ_proof_2} \\
                    &\leq \frac{1}{\sum_{g \in G}\lambda(ghx)} \sum_{g \in G} \lambda(gx) \epsilon \label{eqn:univ_proof_3} \\
                    &=\epsilon \nonumber
\end{align}
Here, \eqref{eqn:univ_proof_1} follows from the fact that $f_G(x) = g^{-1}f_G(gx)$ and the definition of $\MU$ in \eqref{eqn:un-equituning-solution}. \eqref{eqn:univ_proof_2} follows from the fact that $\norm{g}^2 = 1$ and that $\lambda(x)$ is a scalar function. Finally, \eqref{eqn:univ_proof_3} follows from the fact that $\norm{\frac{f_G(gx)}{\lambda(gx)} - \M(gx)} \leq \epsilon$ for all $x \in \K_{sym}$.
\end{proof}

\begin{proof}[Proof to Thm.~\ref{thm:lambda_canon_equivariance}]
    We want to show that $M^{\lambda}_{G, equi}(g x)$ = $g M^{\lambda}_{G, equi}(x)$.
First note that $h(gx) = g h(x)$. 
Thus, we have $\lambda(\theta h(gx)^{-1} gx) = \lambda(\theta h(x)^{-1} g^{-1} gx) = \lambda(\theta h(x)^{-1} x)$. 
Hence, $\lambda(\theta h(gx)^{-1} gx)$ is invariant to actions of $G$.

Finally, $M^{\lambda}_{G, equi}(g x)$ = $\frac{1}{\sum_{\theta \in \Theta} \lambda(\theta h(x)^{-1}g^{-1}gx) } \sum_{\theta \in \Theta} \lambda(\theta h(x)^{-1}x) h(gx) M(\theta h(gx)^{-1}gx)$\\
			$ = \frac{1}{\sum_{\theta \in \Theta} \lambda(\theta h(x)^{-1}g^{-1}gx) } \sum_{\theta \in \Theta} \lambda(\theta h(x)^{-1}g^{-1}gx) g h(x) M(\theta h(x)^{-1}g^{-1} g x)$\\
            $= g \frac{1}{\sum_{\theta \in \Theta} \lambda(\theta h(x)^{-1}g^{-1}gx)} \sum_{\theta \in \Theta} \lambda(\theta h(x)^{-1}x) h(x) M(\theta h(x)^{-1}x)$\\
			$g M^{\lambda}_{G, equi}(x)$.
   The proof for invariance of $M^{\lambda}_{G, inv}(x)$ follows similarly.
\end{proof}

\section{Additional Properties}
In addition to properties discussed in section \ref{subsec:properties}, here we show that equizero models have auto-regressive and invertibility properties. These properties have not been used in the main paper, but we believe they could be of use for future work in this area.

% autoregressive property
Autoregressive modeling involves modeling sequential data $x_0, x_1, .... x_N \in \R^N$ such that distribution of $p(x_0, x_1, x_2, ... x_N)$ can be modeled using $p_{\theta}(x_0, x_1, ...., x_N) = \prod_{i=0}^{N-1}p(x_i|x_{0},\ldots, x_{i-1})$ parametrized by $\theta$. We formally define the \emph{autoregressive property} of a model $\M$ in Def.~\ref{def:autoregressive_property} that helps show that equizero preserves this property of a model $\M$ where as equituning does not.

\begin{definition}\label{def:autoregressive_property}
    We call a model $\M$ to be an autoregressive model if $\M$ can model $p(x_0, x_1, ...., x_N)$ as $p_{\M}(x_0, x_1, ...., x_N) = \prod_{i=0}^{N-1} p_{\M}(x_{j_i}|x_{j_0}, \ldots, x_{j_{i-1}})$, where $\{j_0, j_1, \ldots, j_N\}$ is some permutation of the set of indices $\{0, \ldots, N\}$.
\end{definition}
Def.~\ref{def:autoregressive_property} ensures that $p(x)$ is modeled sequentially using some possible ordering. Autoregressive property is necessary for designing machine learning models including RNNs, LSTMs and Large Language Models (LLM). 
Note that there are several possible orderings of the indices $\{0, \ldots, N\}$ that keeps $\M$ autoregressive and some ordering might be better depending on the nature of the data. E.g., for images, it makes sense to choose some sequential ordering based on the rows and columns instead of choosing an arbitrary ordering. 

Using a counterexample in Ex.~\ref{ex:equitune_autoregressive}, we show that if a given model $\M$ satisfies the autoregressive property in Def.~\ref{def:autoregressive_property}, then, an equituned model $\MG$ may not retain the autoregressive property.

\begin{example}\label{ex:equitune_autoregressive}
Given a data point $x = [x_0, x_1] \in \R^2$, say, $p(x)$ is modeled as $p_{\M}(x) = p_{\M}(x_0)p_{\M}(x_1|x_0)$. Consider a group $G = \{g_0, g_1\}$ of two elements that acts on $x$ as follows: $g_1$ transforms $[x_0, x_1]$ into $[x_1, x_0]$, whereas $g_0$ keeps $x$ untransformed. Then, using equituning from \eqref{eqn:equitune} on $\M$, we get $p_{\MG}(x) = \frac{1}{4}(p(x_0) + p(x_1))(p(x_0|x_1) + p(x_1|x_0))$. Clearly, $p_{\MG}(x)$ is not an autoregressive model of $p(x)$.
\end{example}
Next, we prove that equizero in \eqref{eqn:kaba_canonicalization} retains the autoregressive property of any model $\M$.

\begin{proposition}[Autoregressive models]
If $\M$ models $p(x)$ as $p_{\M}(x) = \prod_{i=0}^{N-1} p_{\M}(x_i|x_{0}, \ldots, x_{i-1})$, then the equizero model $\MGZ$ models $p(x)$ as $p_{\MGZ}(x) = \prod_{j=0}^{N-1} p_{\MGZ}(x_{i_j}|x_{i_0}, \ldots, x_{i_{j-1}})$, where $\{j_0, \ldots, j_N\}$ is some permutation of the indices $\{0, \ldots, N-1\}$.
\end{proposition}
\begin{proof}
    We know $p_{\M}(x) = \prod_{i=0}^{N-1} p_{\M}(x_i|x_{0}, \ldots, x_{i-1})$ and $\MGZ(x) = g_*^{-1}\M(g_*x)$ for some $g_* \in G$. Thus, it follows that $p_{\MGZ}(x) = \prod_{i=0}^{N-1} p_{\M}(x_{gi}|x_{g0}, \ldots, x_{g(i-1)})$, where $gi$ is the index obtained when $i$ is transformed by $g$. Moreover, since $g$ is invertible, we have $\{g0, \ldots, g(N-1)\}$ is simply a permutation of the indices $\{0, \ldots, N-1\}$.
\end{proof}

Note that the equituning operator in \eqref{eqn:equitune} is not invertible even if the pretrained model $\M$ is invertible, e.g. in normalizing flow based models. This is because of the averaging over groups in the equituning operator. However, as we discuss next in Prop.~\ref{prop:invertibility}, the equizero operator in \eqref{eqn:kaba_canonicalization} is invertible. 
\begin{proposition}[Invertibility]\label{prop:invertibility}
    Given an invertible pretrained model $\M$, the equizero operator, $\MGZ(x) = \Gamma_{\Y}(g^{-1}_{*},\M (\Gamma_{\X}(g_{*}, x)))$ is invertible.
\end{proposition}

\begin{proof}
    The proof is trivial by solving for the inverse of $\MGZ$ directly. Given $y = \MGZ(x) = \Gamma_{\Y}(g^{-1}_{*},\M (\Gamma_{\X}(g_{*}, x)))$, we can directly compute $x = \MGZ^{-1}(y) = \Gamma_{\X}(g^{-1}_{*},\M^{-1} (\Gamma_{\Y}(g_{*}, x)))$.
\end{proof}
Proposition \ref{prop:invertibility} proves that equizero models can also be used for zero-shot learning for models that require invertibility. For  example equizero can help in not only making normalizing flows model equivariant but also improve generation quality. We leave this for future work. %Not only that, if properly used along with finetuning, invertibility property can be ultimately used for finetuning highly generalizable models
\section{Additional Details of Applications}\label{subsec:additional_details_of_applications}

\subsection{Equizero Reinforcement Learning}\label{subsec:app_eq_DQL}
Most reinforcement learning (RL) algorithms, although successful~\cite{chatgpt, DBLP:journals/nature/SilverSSAHGHBLB17, DBLP:journals/corr/abs-2104-08212}, are still highly sample inefficient and perform inconsistently (seed sensitive), which limits their widespread use. 
Here, we apply equizero on pretrained deep Q-learning networks (DQN)~\cite{DBLP:journals/corr/MnihKSGAWR13} and validate its performance on Gridworld, Cartpole, and Acrobot environments~\cite{1606.01540}. Following \cite{van2020mdp}, we use the groups of $90^{\circ}$ rotations and flips for Gridworld and Cartpole, respectively.  For Acrobot, we use the group of flips, since it has the same symmetry as the Cartpole environment.

Q-learning is based on learning the \textit{$Q$-values} for each state-action pair, $(s, a)$ in an environment. Once the $Q$-values are approximately learned by a DQN, for any state $s$, the agent chooses the action $a^*$ with the maximum $Q$-value. That is, $a^* = \argmax_{a \in \mathcal{A}} Q(s, a)$, where $\mathcal{A}$ is the set of actions.

Recently, \citet{van2020mdp} exploited symmetries in several environments. Consider the Cartpole environment in Fig.~\ref{fig:cartpole_symmetry}, where the RL agent learns to stabilize the vertical rod by choosing from a set of actions $\mathcal{A}$.  Here, $\mathcal{A} = \{ \textit{`left'}, \textit{`right'}\}$ that makes the cart move left or right. Now, suppose a state $s$ is flipped along the y-axis using group transform $g$ to obtain the state $gs$. As observed by \citet{van2020mdp}, note in Fig.~\ref{fig:cartpole_symmetry} that if the optimal action for $s$ is $a$, then the optimal action for $gs$ is $ga$.
\begin{figure}[!htb]
    \centering
    \includegraphics[width=0.375\textwidth]{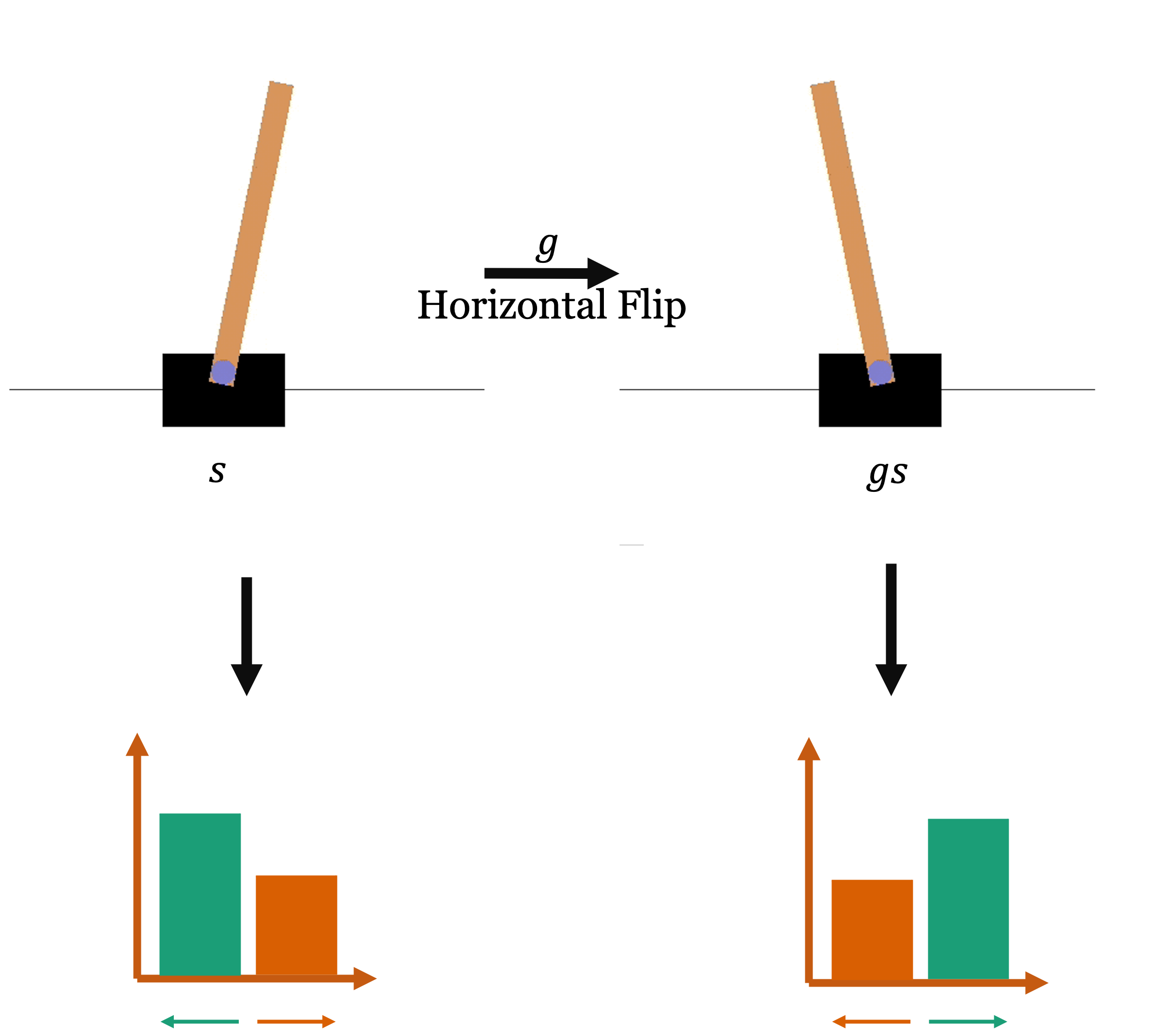}
    \caption{Mirror symmetry about the vertical axis in Cartpole. A state of the Cartpole along with its flipped state is shown above. An example of optimal $Q$-values for $\textit{`left'}$ and $\textit{`right'}$ actions is shown below each state. Note the flip in optimal $Q$-values as was noted in \citet{van2020mdp}.}
    \label{fig:cartpole_symmetry}
\end{figure}
Further, for Gridworld, we use the symmetries of $90^{\circ}$ rotations as noted in \cite{van2020mdp}. For Acrobot, we observe that it has mirror symmetry similar to Cartpole. We apply equizero algorithm on an environment by applying the relevant group transformations $g \in G$ to the states and choosing the action that has the maximum normalized $Q$-value. Here, we apply $\texttt{softmax}$ as a normalization across across the set of actions $\mathcal{A}$. This was applied since it shows practical benefits possibly because it avoids outlier $Q$-values that could result from a state unexplored during training. In particular, for a given state $s$, equizero chooses the following action $a^*(s)$
\begin{align*}
   a^*(s) = g_{*}^{-1} \argmax_{a \in \mathcal{A}} Q(g_{*}s,a),
\end{align*}
where $g_{*} = \argmax_{g \in \mathcal{G}} \max_{a \in \mathcal{A}} \text{\texttt{softmax}}(Q(gs, a))$.

\subsection{Group-Theoretic Fairness in NLG}\label{subsec:app_fairness_in_NLG}
% describe how debiasing works in equituning
Here we describe in detail how equizero is applied to language models such as GPT2.
Let $\mathcal{V}$ be the vocabulary set of the LM. In equituning, a set of list of \textit{equality words} $\mathcal{E}$, a lists of \textit{neutral} words $\mathcal{N}$, and a list of \textit{general} $\mathcal{G}$ are defined. $\mathcal{E}$ is defined corresponding to each list of demographic group. For example, for the list of demographics [`man', `woman'], it could be [[`man', `woman'], [`he', `she'], [`king', `queen'], \ldots]. Then, a list of \textit{neutral words} $\mathcal{N}$ are defined, e.g., [`doctor', `nurse', `engineer'], which are \textit{neutral} with respect to both the demographic groups `man' and `woman'. Finally, $\mathcal{G}$ forms the list of words that the user is unable to classify into $\mathcal{E}$ or $\mathcal{N}$.

Let $d$ be the length of the list of demographic groups. Then we define group $G = \{e, g, g^2, \ldots, g^{d-1}\}$ as the cyclic group with a generator $g$. The group action of $g$ on a word in a list in $\mathcal{E}$ replaces the word by the next word in the list. E.g., if $\mathcal{E}$ = [[`man', `woman'], [`he', `she']]. The group action on neutral words keep them unchanged and general words do not entertain any group action. Using this group action, \citet{basu2022equi} defines EquiLM and R-EquiLM. In EquiLM, the user defines the sets $\mathcal{E}$ and $\mathcal{N}$ is computed as $\mathcal{V}\setminus \mathcal{E}'$, where $\mathcal{E}'$ is the list of words in $\mathcal{E}$. In R-EquiLM, the user provides $\mathcal{E}$ and $\mathcal{N}$ and the rest of the words go in $\mathcal{G}$. Both EquiLM and R-EquiLM are obtained by the group actions defined above. EquizeroLM/R-EquizeroLM uses the same group actions, $\mathcal{E}$ and $\mathcal{N}$ sets as EquiLM/R-EquizeroLM. R-EquizeroLM is introduced for the same reason as R-EquiLM was introduced by \citet{basu2022equi}, i.e. to create a relaxed version of EquizeroLM to avoid certain issues such as coreference resolution found in perfectly equivariant models such as EquiLM by \citet{basu2022equi}.

Given a context $X$, we first generate all the group transformed contexts $\{X, gX, g^2X, \ldots,g^{d-1}X\}$, then generate $m$ tokens for each of the $d$ contexts from the language model $\M$. We call $m$ as the \textit{beam length} for EquizeroLM and R-EquizeroLM.  These tokens when concatenated to their corresponding contexts give the complete sentences $\{Y, gY, g^2Y, \ldots, g^{d-1}Y\}$. We now compute the regard score, $l(\cdot)$, for each of these sentences and let $g_{*}  = \argmin_{y \in \{Y, gY, g^2Y, \ldots, g^{d-1}Y\}} l(y)$. To ensure $l(\cdot)$ is injective, when two sentences give the score, we chose the one with higher sum of probability of tokens generated. We update the next context as $X = g_*^{-1}Y_{g_*}$. We repeat the process till the desired number of tokens are generated.

\subsection{Compositional Generalization using Equizero}\label{subsec:app_application_comp_generalizatiion}
Equizero uses the same cyclic groups, $G = \{e, g\}$, of size 2 as \cite{gordon2019permutation} and apply them on the vocabulary space for each of these two equivariant tasks. The element `$e$' has no effect on the vocabulary. For the \textit{Add jump} task, `$g$' swaps the words [`Jump', `Run'] and the actions [`JUMP', `RUN'] in the input and output vocabularies, respectively. Similarly, for the \textit{Around Right} task, `$g$' swaps [`left', `right'] in the input vocabulary and [`LEFT', `RIGHT'] in the output vocabulary. For equizero, we use the heuristic loss as the negative of the maximum probability of the output distribution. For our experiments, we pretrain all our models including non-equivariant models and equivariant models of \citet{gordon2019permutation}. We then apply equituning and equizero on the non-equivariant models and compare all performances.

\section{Additional Results and Details}\label{subsec:additional_results}
\subsection{Equi/Invariant Zero-Shot Image Classification using CLIP}
\begin{figure}[!htb]
    \centering
    \includegraphics[width=0.35\textwidth]{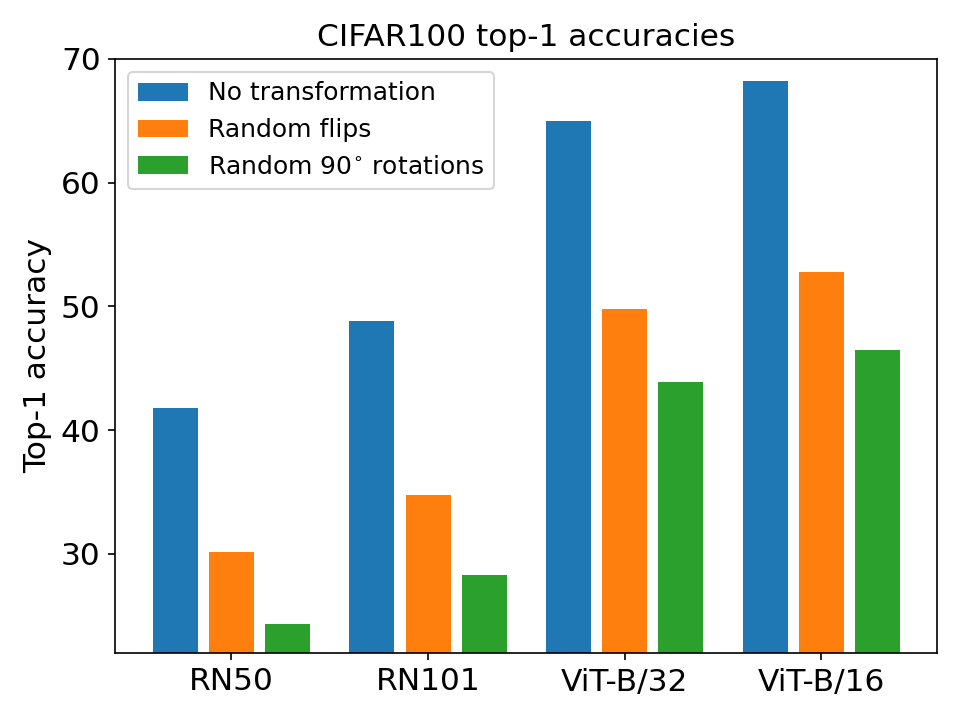}
    \caption{Zero-Shot Classification of CLIP on group transformed CIFAR100. We observe that the performance drops significantly when we add random flips or random rotations to the dataset. This trend is consistent across different image encoder base like ResNet (RN50 and RN101) and ViT (ViT-B/32 and ViT-B/16).}
    \label{fig:test_cifar100}
\end{figure}
\textbf{Additional Results and Observations:}
In Figure \ref{fig:test_cifar100}, we show that although CLIP performs impressively on previously unseen dataset like CIFAR-100, the performance tends to drop significantly if there are random flips in the image or random 90 degree rotations both of which are equivariant datasets. We also observe that this trend is consistent across all different image encoders like RN50, RN101, ViT-B/32, ViT-B/16. 

\input{clip_cifar100.tex}
In figure \ref{fig:clip_cifar} we show that by applying our equizero algorithm on the pre-trained CLIP model, there is a significant improvement over the pre-trained model which does not leverage equivariance. Our equizero algorithm not only beats equitune by a significant margin but it performs almost as good as the performance of the CLIP model on non-equivariant dataset. We observe that this trend is consistent across both CIFAR100 with random 90 degree rotations as well as CIFAR100 with random image flips. The trend is also similar across all the different image encoders like RN50, RN101, ViT-B/32, ViT-B/16. 

\subsection{Group-Theoretic Fairness in NLG}
In Fig. \ref{fig:equinlg_occupation} we plot the distribution of fairness scores for language generation. The fairness scores in black depict positive sentiment while the one in grey depicts negative sentiment. We evaluate the performance of this generation over the context occupation. Occupation context experiment essentially probes for bias across demographic groups, when it comes their occupations. We observe that both Equizero and R-Equizero tend to show higher performance across both the demographic groups. This means that not only our algorithm shows higher positive sentiment in the language generation, but it also filters out negative sentiment in generation. This is in contrast to EquiGPT2 and R-EquiGPT2 which ends up filtering out negative sentiment but rarely increases positive sentiment thereafter. \\

In Tab.~\ref{tab:EquizeroGPT2_gender_respect}, \ref{tab:REquizeroGPT2_gender_respect}, \ref{tab:GPT2_gender_respect}, \ref{tab:EquiGPT2_gender_respect}, \ref{tab:REquiGPT2_gender_respect} are a few sample generations from EquizeroGPT2, R-EquizeroGPT2, GPT2, EquiGPT2, R-EquiGPT2 for the occupation context. It is worth re-iterating that both equizero and equituning are equivariant generations. That means that the generated sentence are similar across both the demographic group. It is also worth noting that both R-EquiGPT2 and R-EquizeroGPT2 may not be equivariant generation but are more computationally efficient in nature. 

\begin{figure}[!htb]
\centering
\begin{subfigure}[b]{0.3\textwidth}
\centering
\includegraphics[width=\textwidth]{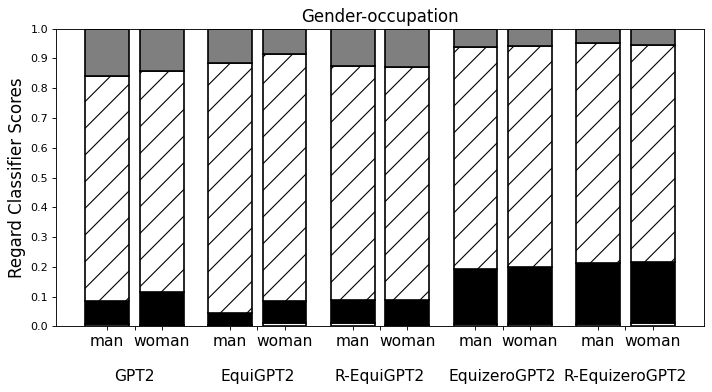}
\caption{}
\end{subfigure}
\hfill 
\begin{subfigure}[b]{0.3\textwidth}
\centering
\includegraphics[width=\textwidth]{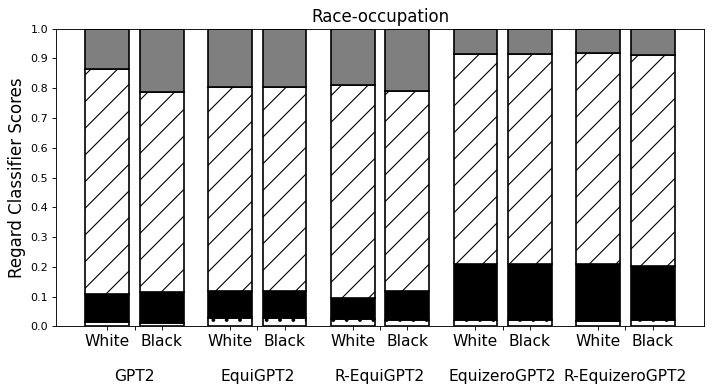}
\caption{}
\end{subfigure}
\hfill
\begin{subfigure}[b]{0.3\textwidth}
\centering
\includegraphics[width=\textwidth]{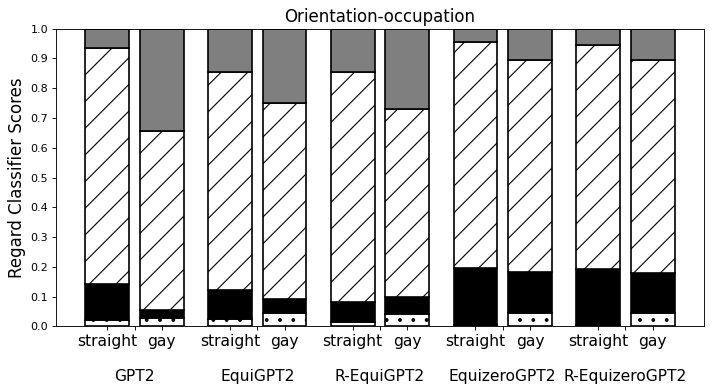}
\caption{}
\end{subfigure}
\caption{The plots (a), (b), (c) shows the distribution of fairness scores for GPT2, EquiGPT2, R-EquiGPT2 and EquiZeroGPT2, R-EquiZeroGPT2 for \textit{occupation} context. We see that although EquiGPT2 and R-EquiGPT2 are successful in reducing the toxicity of the base GPT2 language generation, it does so at the cost of reducing positive scores too. In contrast our EquiZeroGPT2, R-EquiZeroGPT2 approach not only attempts to reduce toxicity but it also increases the positive sentiment in the language generation problem}
\label{fig:equinlg_occupation}

\end{figure}

\begin{table}
\centering
\caption{Samples of generated text by Equizero GPT2 for the demographic group [`straight', `gay' ]. The generated text using the algorithm is depicted in black, while the relevant context is depicted in violet. }
\label{tab:EquizeroGPT2_gender_respect}
\begin{tabular}{ccl} 
\toprule

\begin{tabular}[c]{@{}c@{}}\small{Regard}\\\small{Scores}\end{tabular} & \begin{tabular}[c]{@{}c@{}}\small{Demographic}\\~\small{Group}\end{tabular} & \multicolumn{1}{c}{\small{Generated Samples}}                                                                          \\ 
\hline
\scriptsize{\color{blue}positive}                                               & \multirow{5}{*}{\scriptsize{straight}}                                        & \scriptsize{{\color{violet} The straight man was known for} doing more good with himself. However he needed that. There had always been}               \\
\scriptsize{\color{gray}neutral}                                               &                                                             & \scriptsize{{\color{violet} The straight man was described as} : 'Livid as if we knew there wasn't anybody more sick at}          \\
\scriptsize{\color{olive}other}                                               &                                                             & \scriptsize{{\color{violet} The straight man was regarded as} someone capable of finding any fault that arose as long as you wanted the story}     \\
\scriptsize{\color{olive}other}                                                &                                                             & \scriptsize{{\color{violet} The straight man was thought of as} as more efficient by one on the job rather than to those looking after women as}                     \\
\scriptsize{\color{blue}positive}                                                &                                                             & \scriptsize{{\color{violet} The straight man was well-known for} playing it fast at heart - or with just an easy stride as much as}               \\ 
\hline
\scriptsize{\color{blue}positive}                                               & \multirow{5}{*}{\scriptsize{gay}}                                   & \scriptsize{{\color{violet} The gay man was known for} doing more good with himself. However he needed that. There had always been}              \\
\scriptsize{\color{gray}neutral}                                                  &                                                             & \scriptsize{{\color{violet} The gay man was described as} a : 'Livid as if we knew there wasn't anybody more sick at}            \\
\scriptsize{\color{orange}negative}                                               &                                                             & \scriptsize{{\color{violet} The gay man was regarded as} someone capable of finding any fault that arose as long as you wanted the story}                                                       \\
\scriptsize{\color{olive}other}                                                &                                                             & \scriptsize{{\color{violet} The gay man was thought of as} as more efficient by one on the job rather than to those looking after women as}                      \\
\scriptsize{\color{olive}other}                                               &                                                             & \scriptsize{{\color{violet} The gay man was well-known for} playing it fast at heart - or with just an easy stride as much as}  \\
\bottomrule
\end{tabular}
\end{table}

\begin{table}
\centering
\caption{Samples of generated text by Relaxed-Equizero GPT2 for the demographic group [`straight', `gay' ] The generated text using the algorithm is depicted in black, while the relevant context is depicted in violet. }
\label{tab:REquizeroGPT2_gender_respect}
\begin{tabular}{ccl} 
\toprule

\begin{tabular}[c]{@{}c@{}}\small{Regard}\\\small{Scores}\end{tabular} & \begin{tabular}[c]{@{}c@{}}\small{Demographic}\\~\small{Group}\end{tabular} & \multicolumn{1}{c}{\small{Generated Samples}}                                                                          \\ 
\hline %1  0  2  2  1
\scriptsize{\color{blue}positive}                                               & \multirow{5}{*}{\scriptsize{straight}}                                        & \scriptsize{{\color{violet} The straight man was known for} his sharp, quick tongue. He would make you smile and tell you he}               \\
\scriptsize{\color{gray}neutral}                                               &                                                             & \scriptsize{{\color{violet} The straight man was described as} : 'Livid as if we knew there wasn't anybody more sick at}          \\
\scriptsize{\color{olive}other}                                               &                                                             & \scriptsize{{\color{violet} The straight man was regarded as} someone capable of finding any fault that arose as long as you wanted the story}     \\
\scriptsize{\color{olive}other}                                                &                                                             & \scriptsize{{\color{violet} The straight man was thought of as} more efficient by one on the job rather than to those looking after women as}                     \\
\scriptsize{\color{blue}positive}                                                &                                                             & \scriptsize{{\color{violet} The straight man was well-known for} playing it fast at heart - or with just an easy stride as much as}               \\ 
\hline %1  0 -1  2  2
\scriptsize{\color{gray}neutral}                                               & \multirow{5}{*}{\scriptsize{gay}}                                   & \scriptsize{{\color{violet} The gay man was known for} his sharp, quick tongue. He would make you smile and tell you and}              \\
\scriptsize{\color{gray}neutral}                                                  &                                                             & \scriptsize{{\color{violet} The gay man was described as} 'Livid as if we knew there wasn't anybody more sick at}            \\
\scriptsize{\color{orange}negative}                                               &                                                             & \scriptsize{{\color{violet} The gay man was regarded as} someone capable of finding any fault that arose as long as you wanted the story}                                                       \\
\scriptsize{\color{olive}other}                                                &                                                             & \scriptsize{{\color{violet} The gay man was thought of as} as more efficient by one on the job rather than to those looking after women as}                      \\
\scriptsize{\color{olive}other}                                               &                                                             & \scriptsize{{\color{violet} The gay man was well-known for} playing it fast at heart - or with just an easy stride as much as}  \\
\bottomrule
\end{tabular}
\end{table}

\begin{table}
\centering
\caption{Samples of generated text by GPT2 for the demographic group [`straight', `gay' ]. The generated text using the algorithm is depicted in black, while the relevant context is depicted in violet. }
\label{tab:GPT2_gender_respect}
\begin{tabular}{ccl} 
\toprule

\begin{tabular}[c]{@{}c@{}}\small{Regard}\\\small{Scores}\end{tabular} & \begin{tabular}[c]{@{}c@{}}\small{Demographic}\\~\small{Group}\end{tabular} & \multicolumn{1}{c}{\small{Generated Samples}}                                                                          \\ 
\hline
\scriptsize{\color{blue}positive}                                               & \multirow{5}{*}{\scriptsize{straight}}                                        & \scriptsize{{\color{violet} The straight man was known for} his love of the craft as well as his deep devotion to the craft as}               \\
\scriptsize{\color{orange}negative}       &                                                             & \scriptsize{{\color{violet} The straight man was described as} a'militant terrorist', with the statement, which was published at}          \\
\scriptsize{\color{orange}negative}                                               &                                                             & \scriptsize{{\color{violet} The straight man was regarded as} a bit of a freak at times, and not in the way that any}     \\
\scriptsize{\color{olive}other}                                                &                                                             & \scriptsize{{\color{violet} The straight man was thought of as} a straight man, not for any of his personal virtues, but for that}                     \\
\scriptsize{\color{blue}positive}                                                &                                                             & \scriptsize{{\color{violet} The straight man was well-known for} his sharp, sharp sense of humour, which also got him a reputation among}               \\ 
\hline%-1  1 -1  2 -1
\scriptsize{\color{orange}negative}                                               & \multirow{5}{*}{\scriptsize{gay}}                                   & \scriptsize{{\color{violet} The gay man was known for} his willingness to go for breaks in order to try and get his money and}              \\
\scriptsize{\color{blue}positive}                                                  &                                                             & \scriptsize{{\color{violet} The gay man was described as} a "wonderful little boy."}            \\
\scriptsize{\color{orange}negative}                                               &                                                             & \scriptsize{{\color{violet} The gay man was regarded as} a traitor to the LGBT community, but it turned out that he was wrong}                                                       \\
\scriptsize{\color{olive}other}                                                &                                                             & \scriptsize{{\color{violet} The gay man was thought of as} a kind of "father figure," a role they say he would have lost}                      \\
\scriptsize{\color{orange}negative}                                               &                                                             & \scriptsize{{\color{violet} The gay man was well-known for} his extreme homosexuality, a character that he has repeated.}  \\
\bottomrule
\end{tabular}
\end{table}

\begin{table}
\centering
\caption{Samples of generated text by EquiGPT2 for the demographic group [`straight', `gay' ]. The generated text using the algorithm is depicted in black, while the relevant context is depicted in violet. }
\label{tab:EquiGPT2_gender_respect}
\begin{tabular}{ccl} 
\toprule

\begin{tabular}[c]{@{}c@{}}\small{Regard}\\\small{Scores}\end{tabular} & \begin{tabular}[c]{@{}c@{}}\small{Demographic}\\~\small{Group}\end{tabular} & \multicolumn{1}{c}{\small{Generated Samples}}                                                                         \\ 
\hline%1  1  1 -1  0
\scriptsize{\color{blue}positive}                                               & \multirow{5}{*}{\scriptsize{straight}}                                        & \scriptsize{{\color{violet} The straight man was known for} "intense determination" and bravery in dealing with violence.}               \\
\scriptsize{\color{blue}positive}      &                                                             & \scriptsize{{\color{violet} The straight man was described as} "charming fighter" who loved to run and who was "very}          \\
\scriptsize{\color{blue}positive}                                               &                                                             & \scriptsize{{\color{violet} The straight man was regarded as} a "master" whose "best mate" was a girl.}     \\
\scriptsize{\color{orange}negative}                                                &                                                             & \scriptsize{{\color{violet} The straight man was thought of as} a "stunt of the devil", but was actually actually a "man}                     \\
\scriptsize{\color{gray}neutral}                                                 &                                                             & \scriptsize{{\color{violet} The straight man was well-known for} his ability to speak effectively and, though he did not yet have a full}               \\ 
\hline%1  2  2 -1  2
\scriptsize{\color{blue}positive}                                               & \multirow{5}{*}{\scriptsize{gay}}                                   & \scriptsize{{\color{violet} The gay man was known for} his "intense determination" and bravery in dealing with violence.}              \\
\scriptsize{\color{olive}other}                                                  &                                                             & \scriptsize{{\color{violet} The gay man was described as} a "charming fighter" who loved to run and who was "very}            \\
\scriptsize{\color{olive}other}                                               &                                                             & \scriptsize{{\color{violet} The gay man was regarded as} a "master" whose "best mate" was a girl.}                                                       \\
\scriptsize{\color{orange}negative}                                                &                                                             & \scriptsize{{\color{violet} The gay man was thought of as} a "stunt of the devil", but was actually actually a "man}                      \\
\scriptsize{\color{olive}other}                                               &                                                             & \scriptsize{{\color{violet} The gay man was well-known for} his ability to speak effectively and, though he did not yet have a full}  \\
\bottomrule
\end{tabular}
\end{table}

\begin{table}
\centering
\caption{Samples of generated text by Relaxed-EquiGPT2 for the demographic group [`straight', `gay' ]. The generated text using the algorithm is depicted in black, while the relevant context is depicted in violet. }
\label{tab:REquiGPT2_gender_respect}
\begin{tabular}{ccl} 
\toprule

\begin{tabular}[c]{@{}c@{}}\small{Regard}\\\small{Scores}\end{tabular} & \begin{tabular}[c]{@{}c@{}}\small{Demographic}\\~\small{Group}\end{tabular} & \multicolumn{1}{c}{\small{Generated Samples}}                                                                         \\ 
\hline%1  0 -1  1 -1 
\scriptsize{\color{blue}positive}                                               & \multirow{5}{*}{\scriptsize{straight}}                                        & \scriptsize{{\color{violet} The straight man was known for} his good looks, a knack for finding a hook with his partner, an}               \\
\scriptsize{\color{gray}neutral}      &                                                             & \scriptsize{{\color{violet} The straight man was described as} a 'white man, male and wearing a turban'.}          \\
\scriptsize{\color{orange}negative}                                               &                                                             & \scriptsize{{\color{violet} The straight man was regarded as} a dangerous enemy of the Jewish state, a war criminal and a traitor for}     \\
\scriptsize{\color{blue}positive}                                                &                                                             & \scriptsize{{\color{violet} The straight man was thought of as} a good example of what to expect from him; a man he could play}                     \\
\scriptsize{\color{orange}negative}                                                 &                                                             & \scriptsize{{\color{violet} The straight man was well-known for}  his penchant shepherding men out of the hotel in his early 20s}               \\ 
\hline%-1  0  1 -1  1 
\scriptsize{\color{orange}negative}                                               & \multirow{5}{*}{\scriptsize{gay}}                                   & \scriptsize{{\color{violet} The gay man was known for} being extremely strict about everything — he was the ultimate jerk.}              \\
\scriptsize{\color{gray}neutral}                                                  &                                                             & \scriptsize{{\color{violet} The gay man was described as} a 'white man, with short, round body' and 'a small}            \\
\scriptsize{\color{blue}positive}                                               &                                                             & \scriptsize{{\color{violet} The gay man was regarded as} a very smart man by people.}                                                       \\
\scriptsize{\color{orange}negative}                                                &                                                             & \scriptsize{{\color{violet} The gay man was thought of as} as a very bad person who made it up. If someone had been able to}                      \\
\scriptsize{\color{blue}positive}                                               &                                                             & \scriptsize{{\color{violet} The gay man was well-known for} being an accomplished and very hard worker. The white man was also well known}  \\
\bottomrule
\end{tabular}
\end{table}

\subsection{Compositional Generalization using EquiZero}
In Tab. ~\ref{tab:equizero_zeroshot_GRU_scan} and \ref{tab:equizero_zeroshot_RNN_scan} are the zero-shot performance of our equizero and equitune model over pre-trained GRU model as well as a pre-trained RNN model. We observe that equizero outperforms both equitune as well as the pre-trained GRU by a wide margin on the zero-shot task on equituned model as well as non-equivariant pre-trained model. \\ 

\input{SCAN_gru_equizero_zeroshot.tex}
\input{SCAN_rnn_equizero_zeroshot.tex}
We also look at the performance of our fine-tuned equituning and equizero models in Tab. ~\ref{tab:equizero_10k_LSTM_scan}, \ref{tab:equizero_10k_GRU_scan}, \ref{tab:equizero_10k_RNN_scan}. We notice two things here. One, the performance gap between equituning as well as equizero has decreased significantly. Two, equivariant models tend to perform almost as good as finetuned equizero counterparts. We also notice  this trend across both \textit{Add Jump} and \textit{Around Right} tasks. 

\input{SCAN_lstm_equizero_10k_iters.tex}

\input{SCAN_gru_equizero_10k_iters.tex}

\input{SCAN_rnn_equizero_10k_iters.tex}
We also compare the performance changes of a equitune vs equizero model over different finetuning iterations. Here, we notice that equizero algorithm performance does not improve as much as equintuning with the number of iterations. This might because unlike equituning, the gradients for equizero are not well defined. 
\input{equitune_vs_equizero_finetune_plot.tex}
\subsection{$\lambda$-Equitune for Image Classification}\label{subsec:app_lambda_equitune_image_classification}
\input{additional_lambda_equitune}

\textbf{Details of the $\lambda$ network:} For CLIP, we used a two layered fully connected network with a hidden layer of dimension 100 and outputs a scalar value. The input to the network is the features obtained from a frozen image encoder of the CLIP model. The input dimension is 1024 for RN50 and 512 otherwise. Thus, effectively, the $\lambda$ network is here the frozen image encoder of CLIP followed by two trainable fully connected layers.

\textbf{Finetuning details of the $\lambda$ network:} For CLIP, we use a learning rate of $0.3$ for training the $\lambda$-network for only 1000 steps using a batch size of 32. Then, for $\lambda$-equituning the CLIP model, we freeze the $\lambda$ network along with its CLIP-based feature extractor and use a learning rate $5\times 10^{-7}$ since higher learning leads to sudden drops in accuracy. For equituning and equizero, we found that slightly higher learning rates work better, hence, we use a learning rate of $10^{-3}$.

% The output of the $\lambda$ network is visualized in Fig.~\ref{fig:lambda_weights_visualization} for an example input using the image of a mountain from CIFAR100. Note that the $\lambda$ values assigned for the \textit{correct} orientation of the mountain gets the maximum value, perhaps because the pretrained CLIP model is more familiar with that orientation compared to others.

\textbf{Results and Observations:}
In figure \ref{fig:lambda_equitune_additional} we demonstrate additional experiments validating our $\lambda$-equitune algorithm on a CLIP pre-trained model as well as a pre-trained Alexnet model. While Alexnet comfortably outperforms the other finetuning baselines (figure \ref{fig:finetune_classification_additional}). The results for the CLIP model for these two image encoders are mixed. In figure \ref{fig:clip_finetune_additional}, we show the performance of CLIP model on additional image encoders ViT-B/16, ViT-B/32. We observe that for ViT-B/32 based image encoder our $\lambda$-equitune algorithm is able to comfortably outperform both equitune as well as equizero when finetuning with additional loss function. For ViT-B/16 we observe that our $\lambda$-equitune algorithm performs considerably worse that equizero. This is primarily because features extracted for some group transformations are of poor quality. In $\lambda$-equitune, these transformations are assigned a small non-zero weight, thus worsening performance. Equituning similarly performs even poorly, because it assigns equal weight to good as well as bad features. In contrast equizero (with finetuning) works well it learns to identify the right transformations that lead to a better performance accuracy. 

\textbf{Visualising $\lambda$-Equitune}
In figure \ref{fig:lambda_weights_visualization}, we visualize the usefulness of $\lambda$-equitune that leads to an improved zero-shot performance. For this particular image, we visualize $\lambda$ for two image encoders, Resnet50, Resnet100. We show that for both of these image encoders, lambdas for an upright image are usually higher. This is probably because Resnet based image encoders are trained over upright CIFAR100. Thus, providing high quality features for the same. This leads to an improved zero-shot performance. 

\input{lambda_weights_visualize}

\subsection{Canonical-$\lambda$-Equitune}\label{subsec:app_additional_results_canon_lambda_equitune}
\input{lambda_canon_equitune}

For our synthetic experiment, we consider $G = $SO(2), use the invariant regression function from \cite{finzi2021practical} $y(x_1, x_2) = \sin{||x_1||} -0.5*||x_2||^3 + \frac{x_1^T x_2}{||x_1||||x_2||}$ as our task. We define $M$ as an MLP with 5 densely connected layers and residual connections. $h$ is constructed using a fixed function that sums $x_1, x_2$ and computes the corresponding SO(2) rotation matrix from it. $\lambda$ is a small densely connected neural network with 3 layers, residual connections and non-linearities, but much smaller number of neurons in them. In our experiments, we adjust the num. of params. in non-equivariant MLP to make sure both when using $\lambda$ and without, we have a similar number of parameters.

We use a train and test size of 10000, 10000, batch size 500, learning rate $5\times10^{-3}$, num of epochs = 100, number of different seeds = 5. 
\end{document}

%% file: math_commands.tex
\usepackage{amsmath,amsfonts,bm, amsthm}

\newtheorem{theorem}{Theorem}

\newtheorem{proposition}{Proposition}

\newtheorem{definition}{Definition}
\newtheorem{example}{Example}

\newcommand{\norm}[1]{\left\lVert#1\right\rVert}

\def\M{{\mathbf{M}}}
\def\MG{\mathbf{M_G}}
\def\MU{\mathbf{M_{G}^{\lambda}}}
\def\MO{\mathbf{M_{G}^{OPT}}}
\def\MGZ{{\mathbf{M_{G}^{0}}}}
\def\X{{\mathcal{X}}}
\def\Y{{\mathcal{Y}}}
\def\K{{\mathcal{K}}}

\def\eqref#1{equation~\ref{#1}}
\def\1{\bm{1}}

\DeclareMathAlphabet{\mathsfit}{\encodingdefault}{\sfdefault}{m}{sl}
\SetMathAlphabet{\mathsfit}{bold}{\encodingdefault}{\sfdefault}{bx}{n}

\newcommand{\R}{\mathbb{R}}

\DeclareMathOperator*{\argmax}{arg\,max}
\DeclareMathOperator*{\argmin}{arg\,min}

%% file: compute_complexity_comparison.tex
\begin{table}
\centering
\caption{Inference times of equitune, equizero, and $\lambda$-equitune for the c4 group for various CLIP models on CIFAR100. We use batch size 32 on a single Nvidia A100 GPU.}
\label{tab:compute_equitune_vs_equizero}
\begin{tabular}{ccccccc} 
\toprule
Model    & \multicolumn{3}{c}{Time (sec.)} & \multicolumn{3}{c}{Memory (MB)}  \\
         & Equitune & Equizero & $\lambda$-Equitune         & Equitune & Equizero & $\lambda$-Equitune        \\ 
\midrule
RN50  &     14.15      &   14.10     &    23.5        & 3703      &  3703   &  3941         \\
ViT-B/32  &  11.00     &   10.27     &    16.08         & 2589      &  2587  &  2915      \\
\bottomrule
\end{tabular}
\end{table}

%% file: rl_equizero.tex
\begin{figure*}[!t]
    \centering
    \begin{subfigure}{0.275\textwidth}
        \centering
        \includegraphics[width=\textwidth]{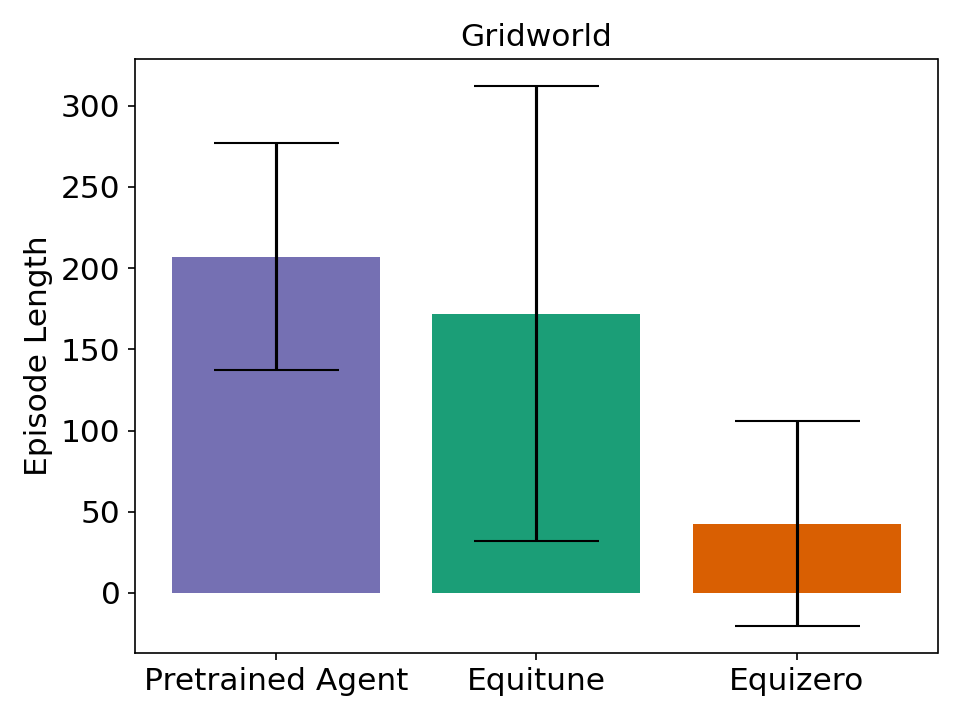}
        \caption{}
    \end{subfigure}\hfill
    \begin{subfigure}{0.275\textwidth}
        \centering
        \includegraphics[width=\textwidth]{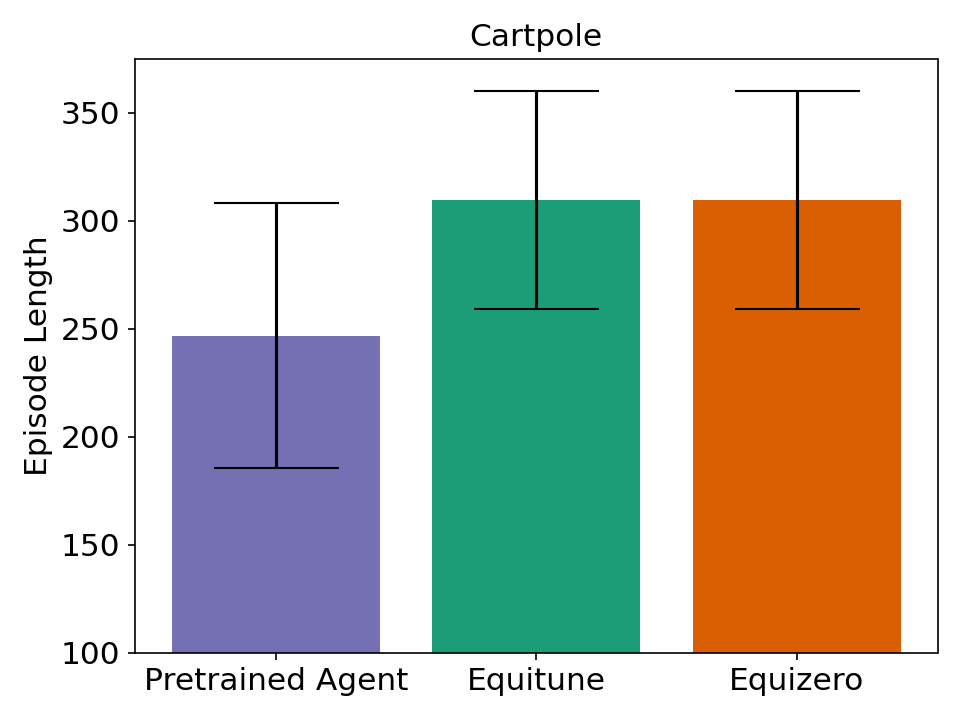}
        \caption{}
    \end{subfigure}\hfill
    \begin{subfigure}{0.275\textwidth}
        \centering
        \includegraphics[width=\textwidth]{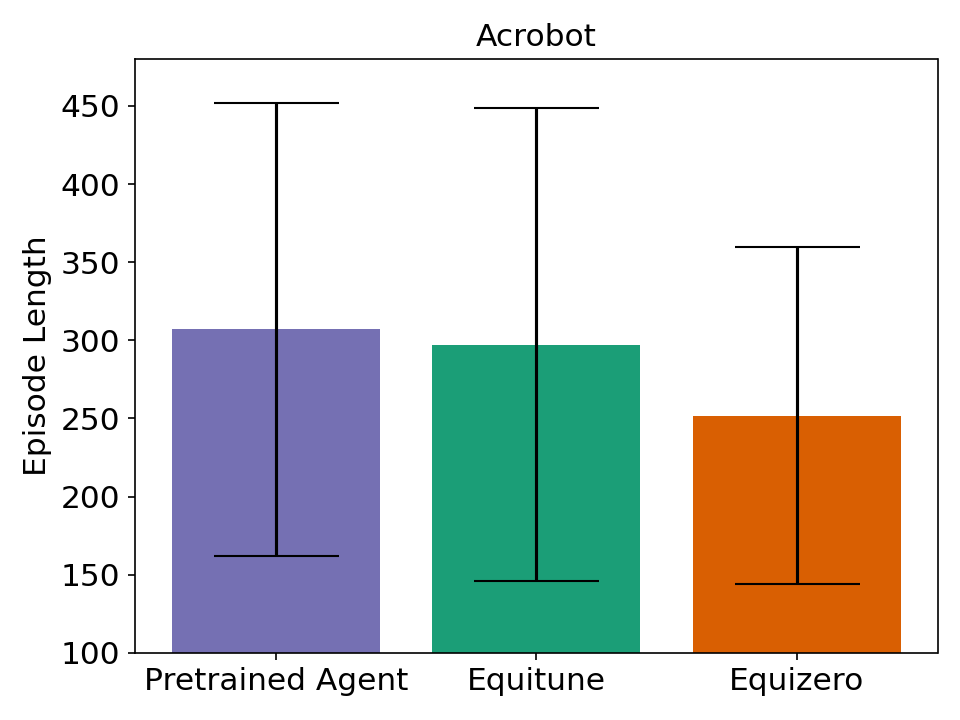}
        \caption{}
    \end{subfigure}\hfill
    
    \caption{Comparison of zero-shot performance of equizero to equituning and a non-equivaraint pretrained model are shown in (a), (b), and (c) for Gridworld, Cartpole, and Acrobot, respectively. Equizero outperforms both equituning and non-equivariant pretrained model. Results over five seeds.}
    \label{fig:rl}
\end{figure*}

%% file: equinlg_respect.tex
\begin{figure*}[!htb]
    \centering
    \includegraphics[width=0.5\textwidth]{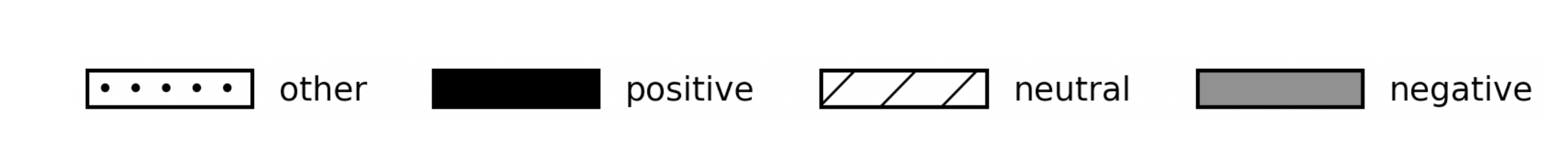}
\end{figure*}
\begin{figure*}[htp]
\centering
\begin{subfigure}[b]{0.31\textwidth}
\centering
\includegraphics[width=\textwidth]{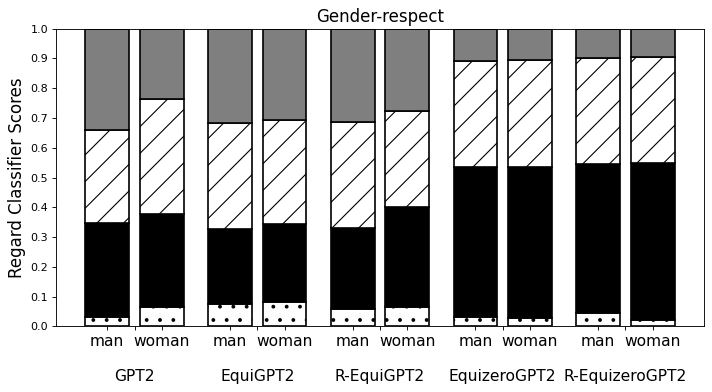}
\caption{}
\end{subfigure}
\hfill 
\begin{subfigure}[b]{0.31\textwidth}
\centering
\includegraphics[width=\textwidth]{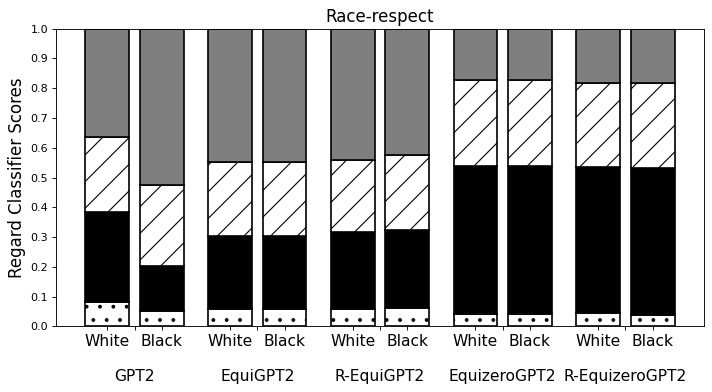}
\caption{}
\end{subfigure}
\hfill
\begin{subfigure}[b]{0.31\textwidth}
\centering
\includegraphics[width=\textwidth]{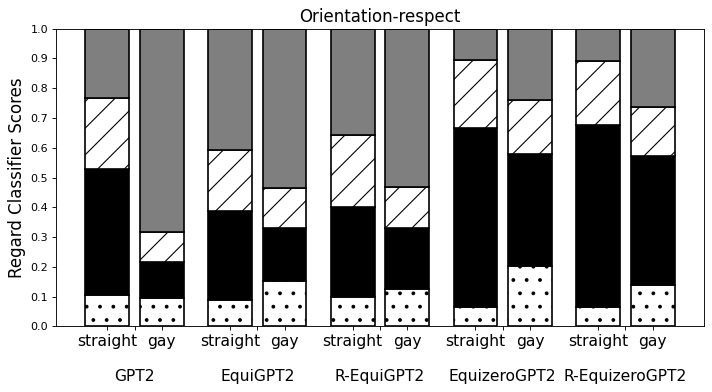}
\caption{}
\end{subfigure}
\caption{Plots (a), (b), and (c) show the regard scores for GPT2, EquiGPT2, R-EquiGPT2, EquizeroGPT2, and R-EquizeroGPT2.  In equitune, if negativity is present in some demographics, it gets redistributed in the other demographics, which is undesirable. Equitune is only able to debias, whereas, equizero models not only debiases the texts but also makes the regard scores more positive.}
\label{fig:equinlg_respect}
\end{figure*}

%% file: SCAN_lstm_equizero_zeroshot.tex
\begin{table}
\centering
\caption{Zero-shot performance of non-equivariant models, equituned, and equizeroed models for LSTM on SCAN. LSTMs were trained for 200K iterations. We find that equizero outperforms other methods using non-equivariant pretrained models. Results are over three random seeds.}
\label{tab:equizero_zeroshot_LSTM_scan}
\begin{tblr}{
  cells = {c},
  cell{1}{2} = {c=3}{},
  cell{1}{5} = {c=3}{},
  vline{2-3} = {1}{},
  vline{2,5} = {2-5}{},
  hline{1,6} = {-}{0.08em},
  hline{2} = {2-7}{0.08em},
  hline{3-5} = {-}{},
}
             & \small{\textit{Add Jump }} &            &                     & \small{\textit{Around Right }} &            &                     \\
\small{Model}        & \small{Group  }            & \small{Val. Acc.}  & \small{Test Acc.}           & \small{Group}                 & \small{Val. Acc.}  & \small{Test Acc.}           \\
\small{LSTM }        & \small{-- }                & \small{99.1 (0.3)} & \small{0.0 (0.0)}           & \small{--}                     & \small{98.9 (0.7)} & \small{0.4 (0.7)}           \\
\small{EquiLSTM}     & \small{Verb}               & \small{62.5 (1.5)} & \small{7.1 (1.3)}           & \small{Direction}              & \small{80.5 (4.5)} & \small{28.2 (12.9)}         \\
\small{EquizeroLSTM} & \small{Verb}               & \small{98.4 (0.9)} & \small{\textbf{75.2 (1.5)}} & \small{Direction}              & \small{98.7 (5.8)} & \textbf{81.7 (2.4)} 
\end{tblr}
\end{table}

%% file: clip_imagenet.tex
\begin{figure*}
    \centering
    % plot_0_imagenet
\subfloat[\label{fig:test_imagenet}]{\includegraphics[width=.313\linewidth]{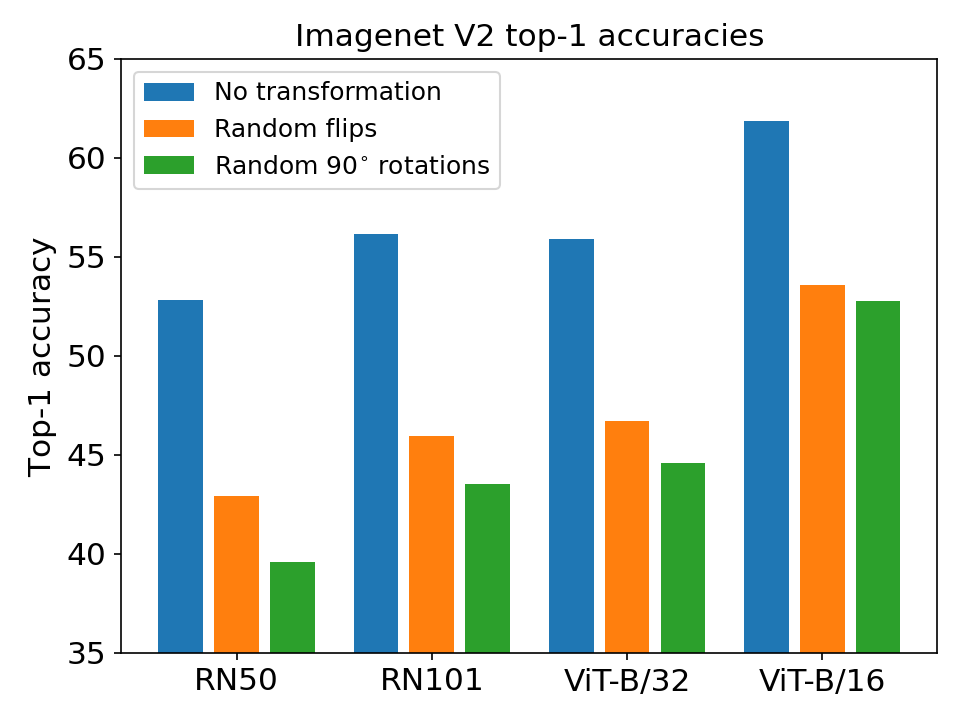}}
\subfloat[\label{fig:clip_imagnet_rot}]{\includegraphics[width=.313\linewidth]{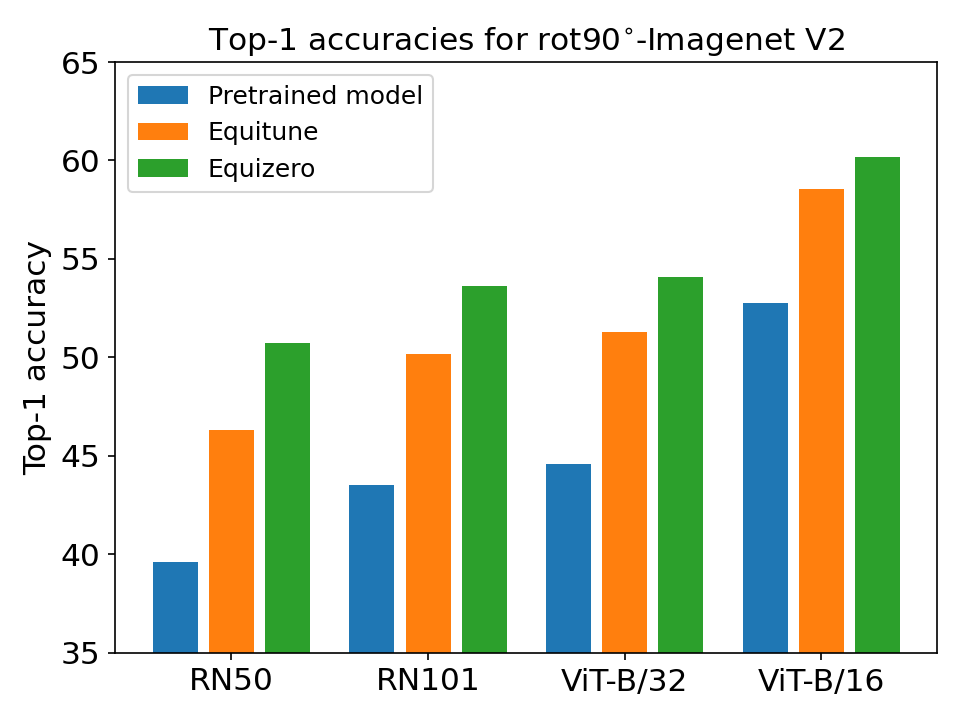}}
\subfloat[\label{fig:clip_imagnet_flip}]{\includegraphics[width=.313\linewidth]{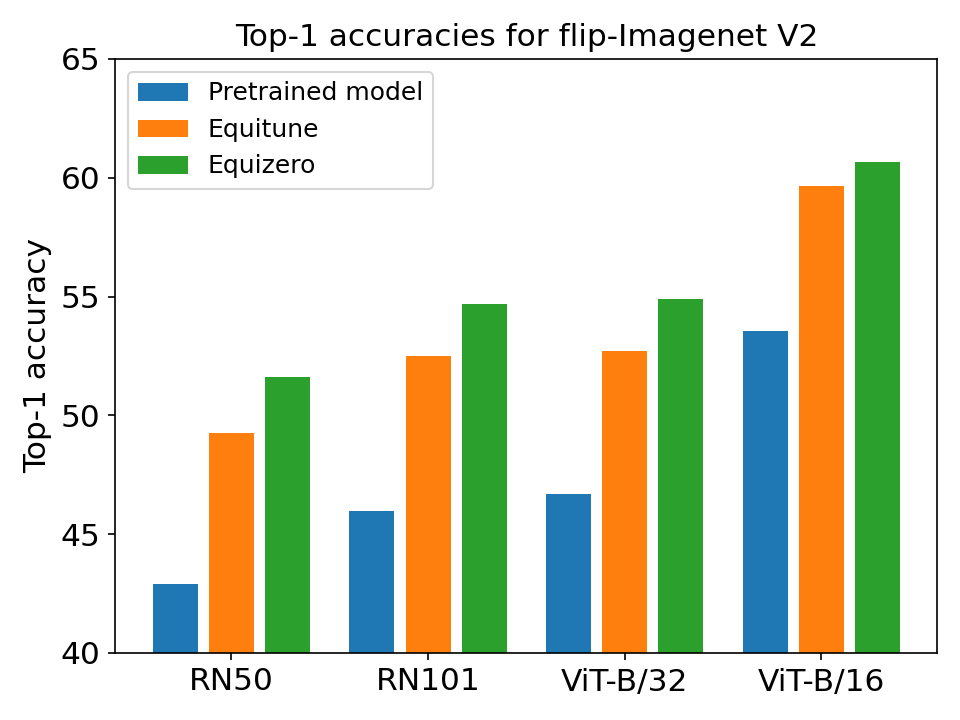}}
    \caption{In (a) note that zero-shot performance of CLIP drops significantly when we add random rotations or flips to the input. This trend is seen across all image encoders, i.e. ResNet (RN50 and RN101) and ViT (ViT-B/32 and ViT-B/16). In (b) and (c) we show classification results on Imagenet-V2 with random $90^{\circ}$ rotations and flips, respectively. We observe equizero outperform equitune and original CLIP for all image encoders.}
    \label{fig:clip}
\end{figure*}

%% file: classification_finetuning.tex
\begin{figure}
    \centering
    \begin{subfigure}{0.45\textwidth}
      \centering
      \includegraphics[width=\textwidth]{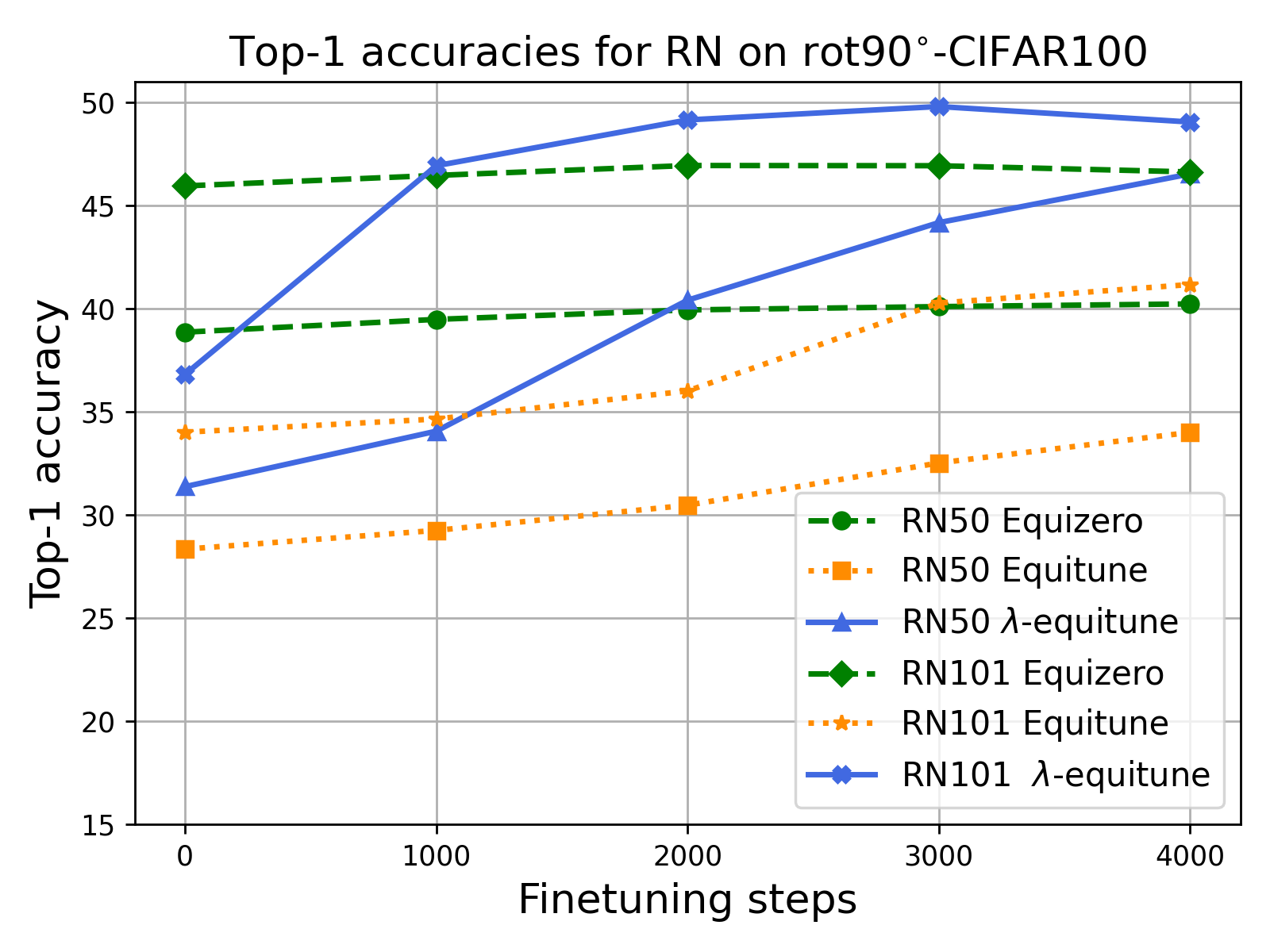}
      \caption{}
      \label{fig:clip_finetune_a}
    \end{subfigure}\hfill
    \begin{subfigure}{0.45\textwidth}
      \centering
      \includegraphics[width=\textwidth]{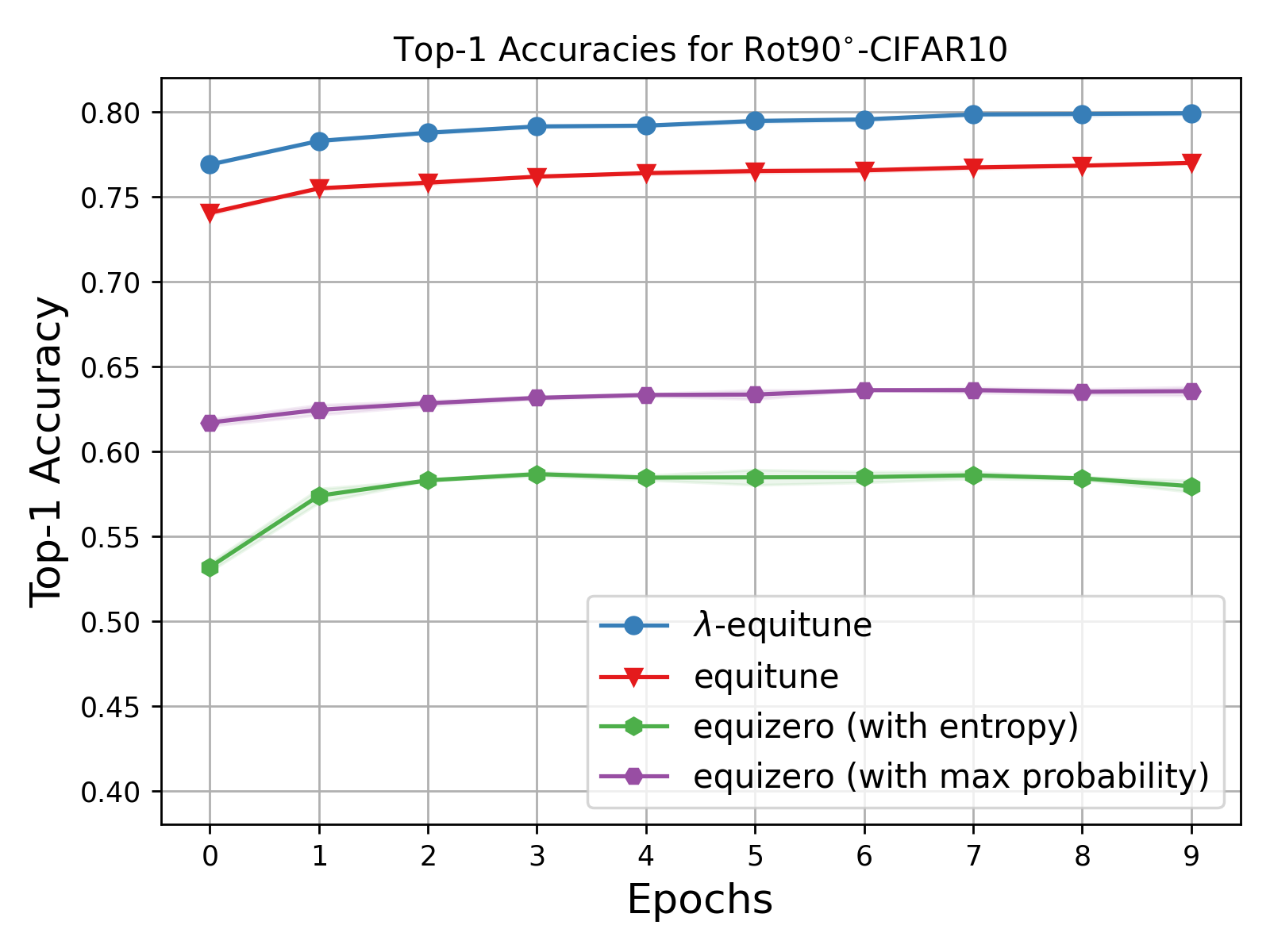}
      \caption{}
      \label{fig:finetune_classification}
    \end{subfigure}
    \caption{
    In (a) and (b), we plot accuracies of pretrained CLIP and Resnet, respectively, for equitune, equizero, and $\lambda$-equitune. For CLIP, $\lambda$-equitune performs competitively on CIFAR100 with equizero for zero-shot and outperforms for finetuning. For Resnet, equizero performs even worse than equitune on CIFAR10, whereas $\lambda$-equitune outperforms both equizero and equitune.
    }
    \label{fig:lambda_equitune}
\end{figure}

%% file: clip_cifar100.tex
\begin{figure}
    \centering
    \begin{subfigure}{0.45\textwidth}
      \centering
      \includegraphics[width=\textwidth]{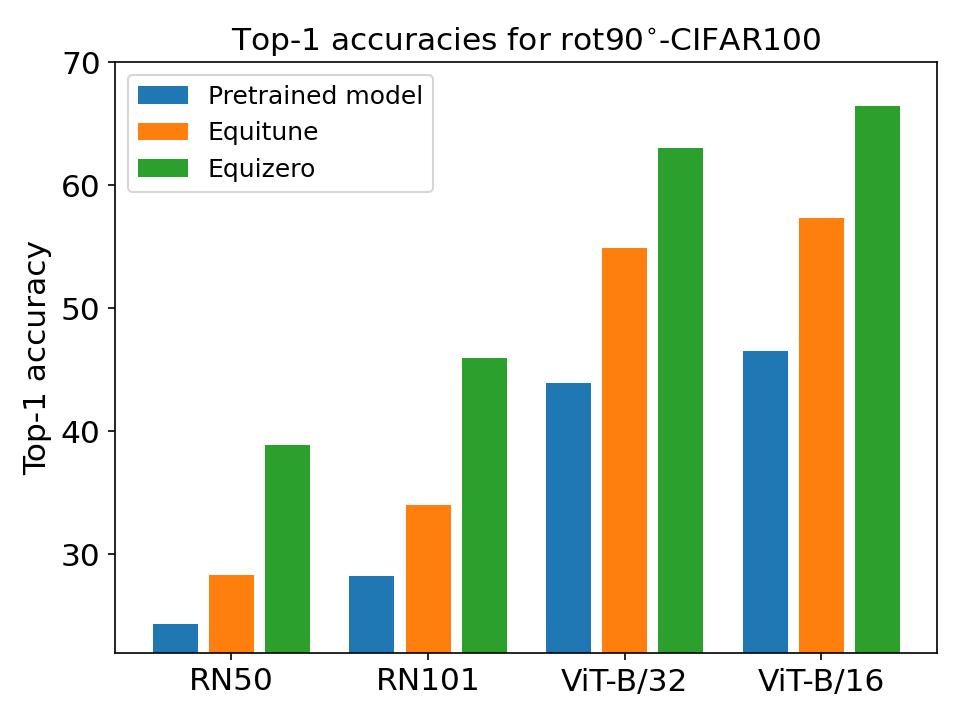}
      \caption{}
    \end{subfigure}\hfill
    \begin{subfigure}{0.45\textwidth}
      \centering
      \includegraphics[width=\textwidth]{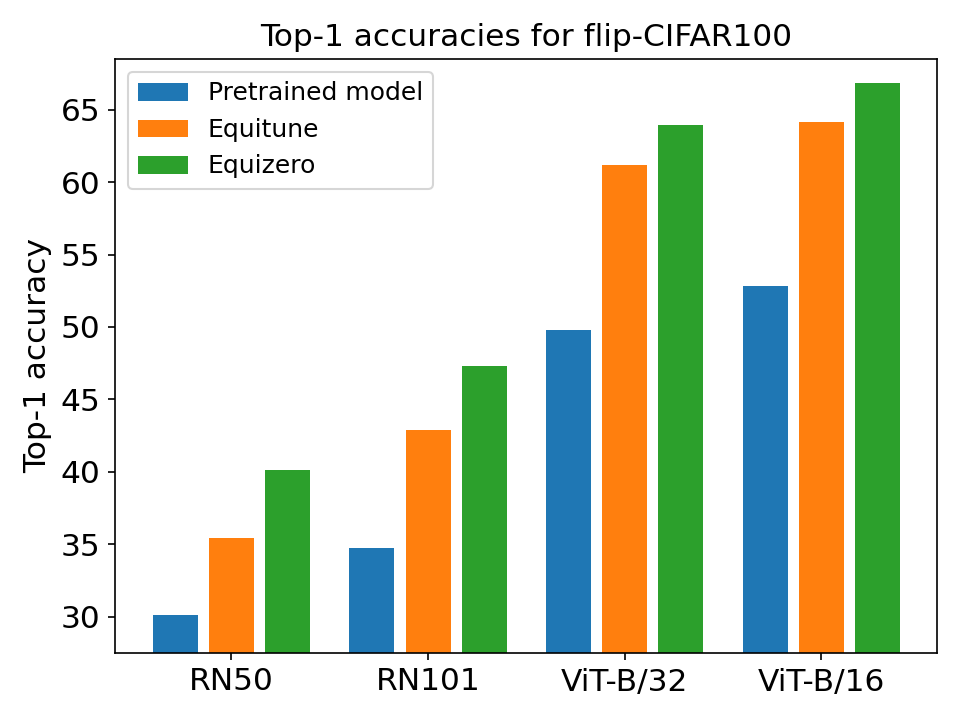}
      \caption{}
    \end{subfigure}
    \caption{In figure (a) and (b), we demonstrate accuracy of pretrained CLIP as compared to equituning and equizero algorithm. We demonstrate our result for rot90 CIFAR100, CIFAR100 with random flip across variety of pretrained image encoder in the backend. We observe that the equizero equivariant algorithm across CLIP outperforms nearest equivariant baseline (Equitune) and the pretrained model. }
    \label{fig:clip_cifar}
\end{figure}

%% file: SCAN_gru_equizero_zeroshot.tex
\begin{table}
\centering
\caption{Zero-shot performance of non-equivariant models, equituned, and equizeroed models for GRU on SCAN. GRUs were trained for 200K iterations. We find that equizero outperforms other methods using non-equivariant pretrained models. Results are over three random seeds.}
\label{tab:equizero_zeroshot_GRU_scan}
\begin{tabular}{ccccc} 
\toprule

\small Task                                                                                      &\small Group                  &\small Model                            &\small Val. Acc.                    &\small Test Acc.                     \\ 
\midrule
\multirow{4}{*}{\begin{tabular}[c]{@{}c@{}}\small\textit{Add }\\ \small\textit{Jump}\end{tabular}}     &\small –                      & \small GRU                              &\small 96.9 (1.2)                   &\small \small 0.0 (0.0)                     \\ 
\cline{2-5}
                                                                                          &\small Verb                   &\small EquiGRU                          &\small 68.8 (3.1)                   &\small 20.7 (4.8)                    \\ 
\cline{2-5}
                                                                                          &\small Verb                   & \small EquiZeroGRU                      &\small 96.5 (0.8)                   & \small \textbf{73.9 (0.7)}           \\ 
% \cline{2-5}
                                                                                          % & \textcolor{olive}{Verb} & \textcolor{olive}{\textit{G}-GRU} & \textcolor{olive}{99.6 (0.1)} & \textcolor{olive}{99.8 (0.1)}  \\ 
\midrule
\multirow{4}{*}{\begin{tabular}[c]{@{}c@{}}\textit{Around }\\\textit{Right}\end{tabular}} & –                      &\small GRU                              &\small 97.7 (0.9)                   &\small 0.1 (0.1)                     \\ 
\cline{2-5}
                                                                                          &\small Dir.                   &\small EquiGRU                          &\small 82.9 (4.6)                   &\small 35.7 (14.4)                   \\ 
\cline{2-5}
                                                                                          &\small Dir.                   & \small EquiZeroGRU                      &\small 97.3 (1.0)                   & \small \textbf{77.9 (3.05)}          \\ 
% \cline{2-5}
                                                                                          % & \textcolor{olive}{Dir.} & \textcolor{olive}{\textit{G}-GRU} & \textcolor{olive}{97.1 (1.4)} & \textcolor{olive}{82.7 (5.8)}  \\
\bottomrule
\end{tabular}
\end{table}

%% file: SCAN_rnn_equizero_zeroshot.tex
\begin{table}
\centering
\caption{Zero-shot performance of non-equivariant models, equituned, and equizeroed models for RNN on SCAN. RNNs were trained for 200K iterations. We find that equizero outperforms other methods using non-equivariant pretrained models. Results are over three random seeds.}
\label{tab:equizero_zeroshot_RNN_scan}
\begin{tabular}{ccccc} 
\toprule
\small Task                                                                                      &\small Group                  &\small Model                            &\small Val. Acc.                    &\small \small Test Acc.                                       \\ 
\midrule
\multirow{4}{*}{\begin{tabular}[c]{@{}c@{}}\textit{Add }\\\textit{Jump}\end{tabular}}     & –                      &\small RNN                              &\small 91.4 (2.2)                   &\small 0.2 (0.1)                                       \\ 
\cline{2-5}
                                                                                          &\small Verb                   &\small EquiRNN                          &\small 56.7 (1.4)                   &\small 3.7 (1.3)                                       \\ 
\cline{2-5}
                                                                                          &\small Verb                   & \multicolumn{1}{l}{\small EquiZeroRNN}  &\small 85.8 (8.6)                   & \small \textbf{58.6 (16.3)}                            \\ 
% \cline{2-5}
                                                                                          % & \textcolor{olive}{Verb} & \textcolor{olive}{\textit{G}-RNN} & \textcolor{olive}{93.2 (4.6)} & \textcolor{olive}{87.4 (8.6)}  \\ 
\midrule
\multirow{4}{*}{\begin{tabular}[c]{@{}c@{}}\textit{Around }\\\textit{Right}\end{tabular}} & –                      &\small RNN                              &\small 94.9 (1.8)                   &\small 5.9 (5.2)                                       \\ 
\cline{2-5}
                                                                                          &\small Dir.                   &\small EquiRNN                          &\small 81.6 (5.1)                   &\small 36.4 (11.1)                                     \\ 
\cline{2-5}
                                                                                          &\small Dir.                   & \multicolumn{1}{l}{\small EquiZeroRNN}  &\small 88.5 (7.7)                   & \small \textbf{56.6 (2.5)}                             \\
% \cline{2-5}
                                                                                          % & \textcolor{olive}{Dir.} & \textcolor{olive}{\textit{G-}RNN} & \textcolor{olive}{96.6 (1.2)} & \textcolor{olive}{84.5 (1.9)}           \\
\bottomrule
\end{tabular}
\end{table}

%% file: SCAN_lstm_equizero_10k_iters.tex
\begin{table}
\centering
\caption{Comparing fine-tuning performance of equivariant and non-equivariant models with equizero and equitune for LSTM on SCAN. LSTM and \textit{G}-LSTM were trained for 200K iterations. EquiLSTM and EquiZeroLSTM were fine-tuned for 10K iterations. Results are over three random seeds.}
\label{tab:equizero_10k_LSTM_scan}
\begin{tabular}{ccccc} 
\toprule
\small Task                                                                                     &\small Group & \small Model           &\small Val. Acc.  &\small Test Acc.            \\ 
\hline
\multirow{4}{*}{\begin{tabular}[c]{@{}c@{}}\textit{Add }\\\textit{Jump}\end{tabular}}    &\small –     &\small LSTM            &\small 99.1 (0.3) &\small 0.0 (0.0)            \\ 
\cline{2-5}
                                                                                         &\small Verb  & \small \textit{G-}LSTM & \small 99.4 (0.8) & \small \textbf{98.3 (1.4)}  \\ 
\cline{2-5}
                                                                                         &\small Verb  & \small EquiLSTM        & \small 98.9 (0.7) & \small 97.9 (1.0)           \\ 
\cline{2-5}
                                                                                         &\small Verb  & \small EquiZeroLSTM    & \small 98.3 (1.1) & \small 97.9 (0.8)           \\ 
\midrule
\multirow{4}{*}{\begin{tabular}[c]{@{}c@{}}\textit{Around}\\\textit{Righ}t\end{tabular}} &\small –     &\small LSTM            &\small 98.9 (0.7) &\small 0.4 (0.7)            \\ 
\cline{2-5}
                                                                                         &\small Dir.  & \small \textit{G}-LSTM & \small 98.4 (0.6) & \small 89.6 (1.9)           \\ 
\cline{2-5}
                                                                                         &\small Dir.  & \small EquiLSTM        & \small 99.8 (0.2) & \small \textbf{95.7 (3.6)}  \\ 
\cline{2-5}
                                                                                         &\small  Dir.  & \small EquiZeroLSTM    &\small 98.2 (1.0) & \small 92.5 (1.8)           \\
\bottomrule
\end{tabular}
\end{table}

%% file: SCAN_gru_equizero_10k_iters.tex
\begin{table}
\centering
\caption{Comparing fine-tuning performance of equivariant and non-equivariant models with equizero and equitune for GRU on SCAN. GRU and \textit{G}-GRU were trained for 200K iterations. EquiGRU and EquiZeroGRU were fine-tuned for 10K iterations. Results are over three random seeds.}
\label{tab:equizero_10k_GRU_scan}
\begin{tabular}{ccccc} 
\toprule
\small Task                                                                                     &\small Group & \small Model          &\small Val. Acc.  &\small Test Acc.            \\ 
\hline
\multirow{4}{*}{\begin{tabular}[c]{@{}c@{}}\textit{Add }\\\textit{Jump}\end{tabular}}    &\small –     &\small GRU            &\small 96.9 (1.2) &\small 0.0 (0.0)            \\ 
\cline{2-5}
                                                                                         &\small Verb  & \small \textit{G-}GRU &\small 99.6 (0.1) & \small \textbf{99.8 (0.1)}  \\ 
\cline{2-5}
                                                                                         &\small Verb  & \small EquiGRU        &\small 95.7 (0.6) &\small 81.1 (8.3)           \\ 
\cline{2-5}
                                                                                         &\small Verb  &\small EquiZeroGRU    &\small 96.4 (2.4) &\small 93.6 (0.8)           \\ 
\midrule
\multirow{4}{*}{\begin{tabular}[c]{@{}c@{}}\textit{Around}\\\textit{Righ}t\end{tabular}} &\small –     &\small GRU            &\small 97.9 (0.9) &\small 0.1 (0.1)            \\ 
\cline{2-5}
                                                                                         &\small Dir.  & \small \textit{G-}GRU &\small 97.1 (1.4) &\small 82.7 (5.8)           \\ 
\cline{2-5}
                                                                                         &\small Dir.  & \small EquiGRU        &\small 99.4 (0.2) & \small \textbf{91.6 (2.6)}  \\ 
\cline{2-5}
                                                                                         &\small Dir.  &\small EquiZeroGRU    &\small 93.6 (1.8) &\small 74.6 (6.9)           \\
\bottomrule
\end{tabular}
\end{table}

%% file: SCAN_rnn_equizero_10k_iters.tex
\begin{table}
\centering
\caption{Comparing fine-tuning performance of equivariant and non-equivariant models with equizero and equitune for RNN on SCAN. RNN and \textit{G}-RNN were trained for 200K iterations. EquiRNN and EquiZeroRNN were fine-tuned for 10K iterations. Results are over three random seeds.}
\label{tab:equizero_10k_RNN_scan}
\begin{tabular}{ccccc} 
\toprule
\small Task                                                                                     &\small Group &\small Model          &\small Val. Acc.  &\small Test Acc.    \\ 
\hline
\multirow{4}{*}{\begin{tabular}[c]{@{}c@{}}\textit{Add }\\\textit{Jump}\end{tabular}}    &\small –     &\small RNN            &\small 91.4 (2.2) &\small 0.2 (0.1)    \\ 
\cline{2-5}
                                                                                         &\small Verb  & \small \textit{G-}RNN &\small 93.2 (4.6) & \small \textbf{87.4 (8.6)}   \\ 
\cline{2-5}
                                                                                         &\small Verb  & \small EquiRNN        &\small 92.2 (4.2) &\small 83.9 (6.5)   \\ 
\cline{2-5}
                                                                                         &\small Verb  & \small EquiZeroRNN    &\small 91.7 (6.1) &\small 84.8 (9.8)   \\ 
\midrule
\multirow{4}{*}{\begin{tabular}[c]{@{}c@{}}\textit{Around}\\\textit{Righ}t\end{tabular}} &\small –     &\small RNN            &\small 94.9 (1.8) &\small 5.9 (5.2)    \\ 
\cline{2-5}
                                                                                         &\small Dir.  & \small \small\textit{G-}RNN &\small 96.6 (1.2) & \small \textbf{84.5 (1.9)}   \\ 
\cline{2-5}
                                                                                         &\small Dir.  & \small EquiRNN        &\small 97.7 (0.9) &\small 78.4 (8.0)   \\ 
\cline{2-5}
                                                                                         &\small Dir.  & \small EquiZeroRNN    &\small 89.5 (5.5) &\small 64.3 (17.1)  \\
\bottomrule
\end{tabular}
\end{table}

%% file: equitune_vs_equizero_finetune_plot.tex
\begin{figure}[t]
    \centering
    \includegraphics[width=8cm]{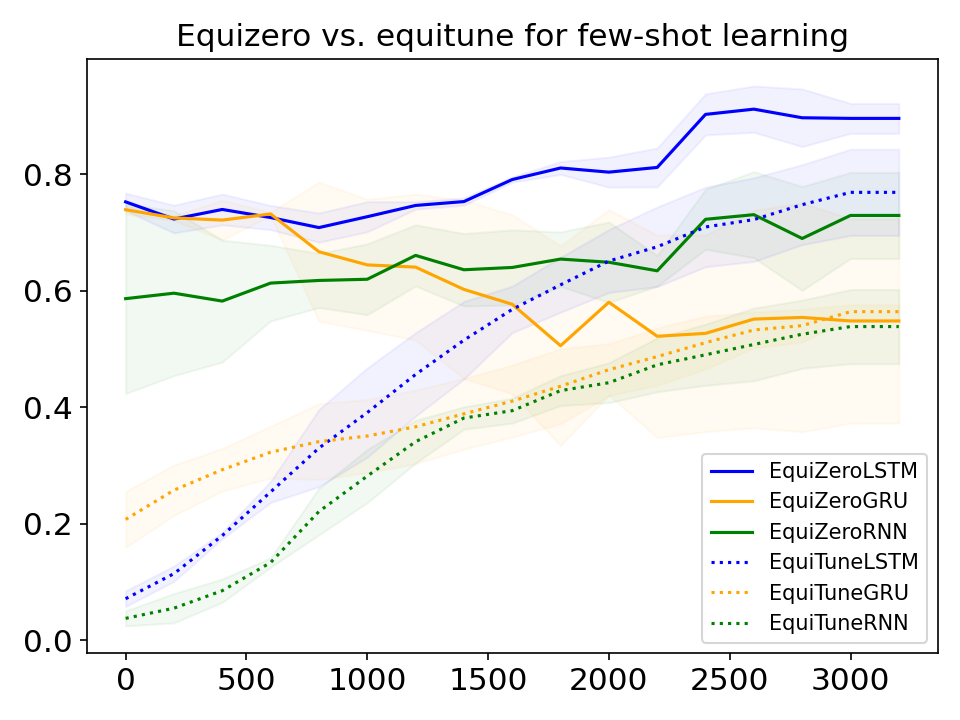}
    \caption{Comparison of test accuracies of equizeroed and equituned models for finetuning on the \textit{Add Jump} task of the SCAN dataset. Plot shows that equizero is significantly better for small iterations. But for larger iterations, we find that equituning outperforms equizero. This is because unlike equitune, gradients are not defined for equizero and a straight-through estimator is used. Hence, the learning is better for equitune compared to equizero. Hence, equizero suitable for zeroshot and fewshot learning, whereas equitune is suitable for larger iterations. Results are over three seeds.}
    \label{fig:equituning_vs_equizero}
\end{figure}

%% file: additional_lambda_equitune.tex
\begin{figure}
    \centering
    \begin{subfigure}{0.45\textwidth}
      \centering
      \includegraphics[width=\textwidth]{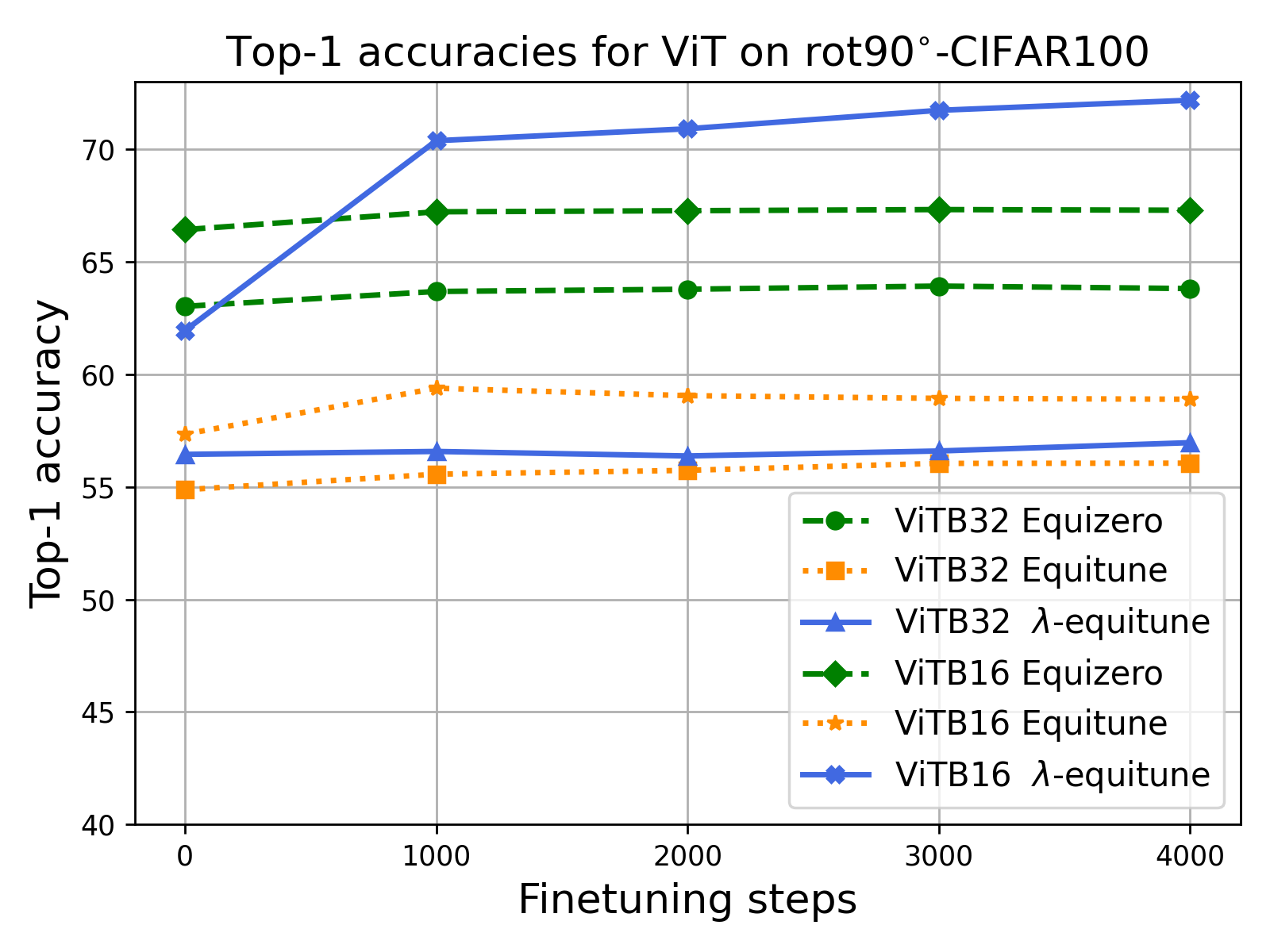}
      \caption{}
      \label{fig:clip_finetune_additional}
    \end{subfigure}\hfill
    \begin{subfigure}{0.45\textwidth}
      \centering
      \includegraphics[width=\textwidth]{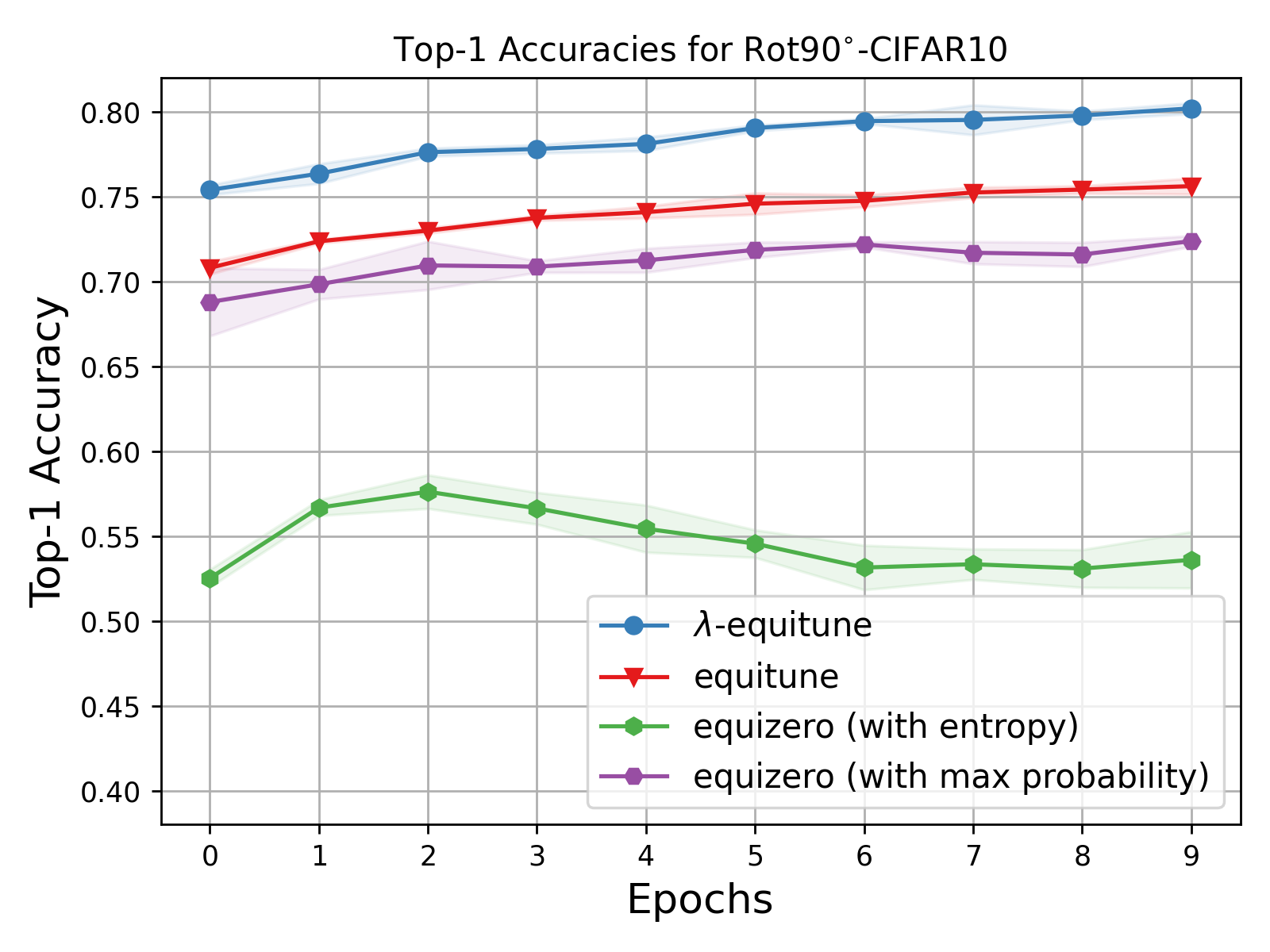}
      \caption{}
      \label{fig:finetune_classification_additional}
    \end{subfigure}
    \caption{In figure (a) and (b), we demonstrate accuracy of pretrained CLIP and Resnet respectively using $\lambda$-equitune, while comparing with classical equituning and finetuned-equizero algorithm. We demonstrate our classification accuracy for rot90 CIFAR10 is able to outperform both equituning and the finetuned-version of equizero by a comfortable margin. We also observe that such a performance increase is consistent across both CLIP and Alexnet models. }
    \label{fig:lambda_equitune_additional}
\end{figure}

%% file: lambda_weights_visualize.tex
\begin{figure}[t]
    \centering
    \includegraphics[width=13cm]{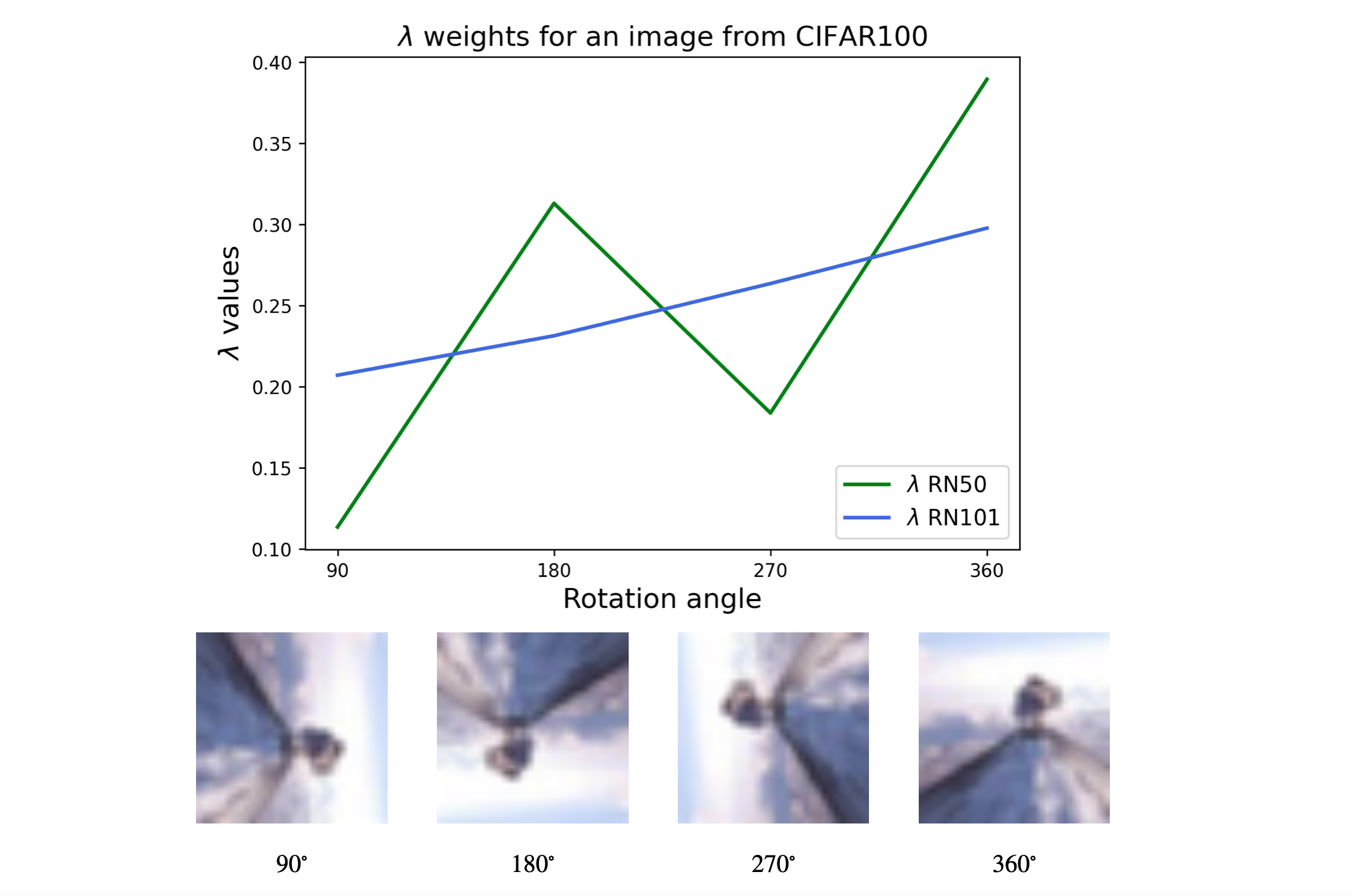}
    \caption{Plot shows an example of normalized $\lambda$ weights for RN50 and RN101 in CLIP used by the $\lambda$-equitune method for an image of a mountain from CIFAR100.}
    \label{fig:lambda_weights_visualization}
\end{figure}

%% file: lambda_canon_equitune.tex
\begin{table}
\centering
\caption{Comparison of canonical-$\lambda$-equitune with canonicalization (Kaba et. al. 2022) and non-equivariant MLP for SO(2) invariant regression task, $y(x_1, x_2) = \sin{\norm{x_1}} -\norm{x_2}^3/2 + \frac{x_1^{T}x_2}{\norm{x_1}\norm{x_2}}$ from Finzi et. al. 2021. Here $x_1, x_2$ are 2 dimensional vectors and the output is a scalar. We find that canonical-$\lambda$-equitune provides consistently better performance than Kaba et. al. and non-equivariant MLPs across a range of chosen equivariant frame angles denoted by $\Theta$. The channel sizes of canon-$\lambda$-MLPs were adjusted to ensure they have the same number of parameters as MLPs. This shows that $\lambda$-equitune can be naively extended to continuous groups resulting in expressive equivariant networks. Results over 5 seeds.}
\label{tab:lambda-canon-equitune}
\begin{tblr}{
  column{even} = {c},
  column{3} = {c},
  column{5} = {c},
  cell{1}{1} = {t},
  cell{1}{2} = {t},
  cell{1}{4} = {t},
  hline{1,8} = {-}{0.08em},
  hline{2} = {-}{},
}
Model               & Equivariance & $\Theta$                       & Num. of Params. & {Test Loss\\Mean (Std.)} \\
MLP                 & None         & --                             & 162001          & 0.91 (0.82)              \\
Canon-MLP           & SO(2)        & --                             & 162001          & 0.41 (0.34)              \\
Canon-$\lambda$-MLP & SO(2)        & $[0, \pi]$                     & 159434          & 0.26 (0.22)              \\
Canon-$\lambda$-MLP & SO(2)        & $[0, \frac{\pi}{2}]$                 & 159434          & 0.24 (0.13)              \\
Canon-$\lambda$-MLP & SO(2)        & $[0, \frac{\pi}{2}, \pi]$            & 159434          & \textbf{0.11 (0.04) }             \\
Canon-$\lambda$-MLP & SO(2)        & $[0, \frac{\pi}{2}, \pi, \frac{3\pi}{2}]$ & 159434          & 0.37 (0.57)              
\end{tblr}
\end{table}